\documentclass[journal]{IEEEtran}
\usepackage{float}
\usepackage{color}
\usepackage{graphicx}
\usepackage[cmex10]{amsmath}
\usepackage{subfigure}
\usepackage{epstopdf}
\usepackage{url}
\usepackage{multirow}
\usepackage{cite}
\usepackage{algorithm}
\usepackage{algorithmic}
\newtheorem{theorem}{Theorem}[section]
\newtheorem{lemma}[theorem]{Lemma}
\newtheorem{proposition}[theorem]{Proposition}
\newtheorem{corollary}[theorem]{Corollary}

\newenvironment{proof}[1][Proof]{\begin{trivlist}
\item[\hskip \labelsep {\bfseries #1}]}{\end{trivlist}}

\newcommand{\qed}{\nobreak \ifvmode \relax \else
      \ifdim\lastskip<1.5em \hskip-\lastskip
      \hskip1.5em plus0em minus0.5em \fi \nobreak
      \vrule height0.75em width0.5em depth0.25em\fi}

\begin{document}
\title{A Confident Information First Principle for Parametric Reduction and Model Selection of Boltzmann Machines}

\author{Xiaozhao~Zhao,
        Yuexian~Hou,
        Dawei~Song,
        and~Wenjie~Li
\thanks{X. Zhao, Y. Hou, D. Song are with the School of Computer Science and Technology, Tianjin University, Tianjin, 300072 China. e-mail: (0.25eye@gmail.com; yxhou@tju.edu.cn; dawei.song2010@gmail.com). D. Song is also with the Department of Computing, The Open University, Milton Keynes, UK.}
\thanks{W. Li is with the Department of Computing, The Hong Kong Polytechnic University, Hung Hom, Kowloon, Hong Kong, China. e-mail: (cswjli@comp.polyu.edu.hk).}
}


\maketitle

\begin{abstract}
Typical dimensionality reduction (DR) methods are often data-oriented, focusing on directly reducing the number of random variables (features) while retaining the maximal variations in the high-dimensional data. In unsupervised situations, one of the main limitations of these methods lies in their dependency on the scale of data features. This paper aims to address the problem from a new perspective and considers model-oriented dimensionality reduction in parameter spaces of binary multivariate distributions.

Specifically, we propose a general parameter reduction criterion, called Confident-Information-First (CIF) principle, to maximally preserve confident parameters and rule out less confident parameters. Formally, the confidence of each parameter can be assessed by its contribution to the expected Fisher information distance within the geometric manifold over the neighbourhood of the underlying real distribution.

We then revisit Boltzmann machines (BM) from a model selection perspective and theoretically show that both the fully visible BM (VBM) and the BM with hidden units can be derived from the general binary multivariate distribution using the CIF principle. This can help us uncover and formalize the essential parts of the target density that BM aims to capture and the non-essential parts that BM should discard. Guided by the theoretical analysis, we develop a sample-specific CIF for model selection of BM that is adaptive to the observed samples. The method is studied in a series of density estimation experiments and has been shown effective in terms of the estimate accuracy.
\end{abstract}

\begin{IEEEkeywords}
Information Geometry, Boltzmann Machine, Parametric Reduction, Fisher Information
\end{IEEEkeywords}

\IEEEpeerreviewmaketitle

\section{Introduction}\label{sec:intro}
\IEEEPARstart{R}{ecently}, deep learning models (e.g., Deep Belief Networks (DBN)\cite{hinton06}, Stacked Denoising Auto-encoder \cite{Vincent2010Denoising}, Deep Boltzmann Machine (DBM) \cite{Salakhutdinov2012} and etc.) have drawn increasing attention due to their impressive empirical performance in various application areas, such as computer vision \cite{Bengio06greedy}\cite{ranzato07autoencoder}\cite{osindero07image}, natural language processing \cite{Collobert08nlp} and information retrieval \cite{Salakhutdinov07kernel}\cite{Salakhutdinov07sigir}.
Despite of these practical successes, there have been debates on the fundamental principle of the design and training of those deep architectures. In most situations, searching the parameter space for deep learning models is difficult.
To tackle this difficulty, \emph{unsupervised pre-training} has been introduced as an important process. In \cite{bengio2010why}, it has been empirically shown that the unsupervised pre-training could fit the network parameters in a region of the parameter space that could well capture the data distribution, thus alleviating generalization error of the trained deep architectures.

The process of pre-training aims to discover the latent representation of the input data based on the learnt generative model, from which we could regenerate the input data. A better generative model would generally lead to more meaningful latent representations.
From the density estimation point of view, pre-training can be interpreted as an attempt to recover a set of parameters for a generative model that describes the underlying distribution of the observed data.
Since Boltzmann machines (BM) are building blocks for many deep architectures (e.g., DBN and DBM), we will focus on a formal analysis of the essential parts of the target density that the BM can capture in model selection.

In practice, the datasets that we deal with are often high-dimensional. Thus we would require a model with high-dimensional parameter space in order to effectively depict the underlying real distribution.
Overfitting usually occur when the model is excessively complex with respect to a small dataset. On the other hand, if a large dataset is available, underfitting would occur when the model is too simple to capture the underlying trend of the data.
Moreover, this connection becomes more complicated if the observed samples contain noises.
Thus, to alleviate overfitting or underfitting, a basic model selection criterion is needed to adjust the complexity of the model with respect to the available observations (usually insufficient or perturbed by noises). Next, for density estimation, we will restate the model selection problem as parametric reduction on the parameter space of multivariate distributions, which could lead to our general parameter reduction criterion, i.e., the Confident-Information-First (CIF) principle.

Assuming there exists an universal parametric probabilistic model $S$ (with $n$ free parameters) that is general enough to represent all system phenomena, the goal of the parametric reduction is to derive a lower-dimensional sub-model $M$ (with $k<<n$ free parameters)
by reducing the number of free parameters in $S$. Note that the number of free parameters is adopted as a model complexity measure, which is in line with various model selection criteria (such as Akaike information criterion (AIC)\cite{aic1974}, Bayesian information criterion (BIC)\cite{bic1978} and etc).

\begin{figure}
  \centering
  \includegraphics[width=0.5\textwidth]{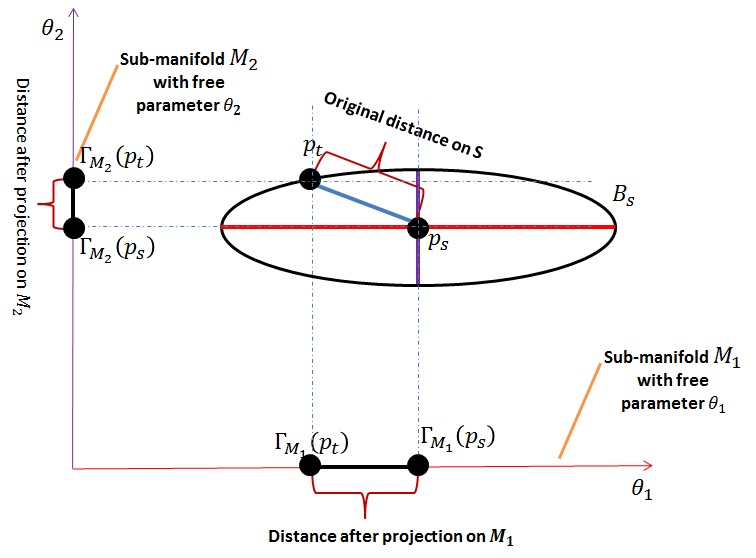}
  \caption{Illustration on parametric reduction: Let $S$ be a two-dimensionality manifold with two free parameters $\theta_1$ and $\theta_2$, and $M_1$ with free parameter $\theta_1$ and $M_2$ with free parameter $\theta_2$ are the submanifold of $S$; As an illustration in Euclidean space, we show $B_s$ (on which the true distribution $p_t$ located on) as the surface of a hyper-ellipsoid centered at sample distribution $p_s$ determined by the Fisher-Rao metric; Only part of the original distance between $p_t$ and $p_s$ ($p_t, p_s \in S$) can be preserved after projection on submanifold $M$; The preferred $M$ is the one that maximally preserves the original distance after projection.
  Note that the scale of the distances in Fig \ref{fig:distanceProjection} are shown as a demo, and are not exactly proportional to the real Riemann distances induced by Fisher-Rao metric}
  \label{fig:distanceProjection} 
\end{figure}
In this paper, we formalize the parametric reduction in the theoretical framework of information geometry (IG). In IG, the general model $S$ can be seen as a $n$-dimensionality manifold and $M$ is a smoothed submanifold of $S$. The number of free parameters in $M$ is restricted to be a constant $k$ ($k<<n$). Then, the major difficulty in the parametric reduction procedure is the choice of parameters to keep or to cut. In this paper, we propose to reduce parameters such that the original geometric structure of $S$ can be preserved as much as possible after projecting on the submanifold $M$.

Let $p_t, p_s\in S$ be the true distribution and the sampling distribution (maybe perturbed from $p_t$ by sampling bias or noises) respectively. It can be assumed that the true distribution $p_t$ is located somewhere in a $\varepsilon$-sphere surface $B_s$ centered at $p_s$, i.e., $B_s = \{p_t \in S | D(p_t, p_s) = \varepsilon\}$, where $D(\cdot, \cdot)$ denotes the distance measure on the manifold $S$, and $\varepsilon$ is a small number. This assumption is made without losing generality, since the $\varepsilon$ is a small variable. For a distribution $p$, the best approximation of $p$ on $M$ is the point $q$ that belongs to $M$ and is the closest to $p$ in terms of the distance measure, i.e., $q = \arg \min_{q'\in M} D(q', p)$, which is defined as the projection of $p$ onto $M$ (denoted as $\Gamma_M(p)$).

Then, the parametric reduction can be defined as the optimization problem to maximally preserve the expectation of the Fisher information distance with respect to the constraint of the parametric number, when projecting distributions from the parameter space of $S$ onto that of the reduced submanifold $M$:
\begin{equation}\label{eq:parametricreduction2}
\begin{aligned}
& \underset{M}{\text{maximize}}
& & \oint_{p_t \in B_s} D(\Gamma_M(p_t), \Gamma_M(p_s)) dB_s \\
& \text{subject to}
& & \text{$M$ has $k$ free parameters}
\end{aligned}
\end{equation}

Here, the Fisher information distance (FID), i.e., the Riemannian distance induced by the Fisher-Rao metric \cite{amari92igbm}, is adopted as the distance measure between two distributions, since it is shown to be the unique metric meeting a set of natural axioms for the distribution metric \cite{Amari93}\cite{gibilisco2010algebraic}\cite{chentsov1982statistical}, e.g., the invariant property with respect to reparametrizations and the monotonicity with respect to the random maps on variables.
Let $\xi$ be the distribution parameters. For two close distributions $p_1$ and $p_2$ with parameters $\xi_1$ and $\xi_2$, the Fisher information distance between $p_1$ and $p_2$ is:
\begin{equation}\label{eq:fisherdistance}
  D(p_1, p_2)=\sqrt{(\xi_1-\xi_2)^T G_\xi (\xi_1-\xi_2)}
\end{equation}
where $G_\xi$ is the Fisher information matrix \cite{Amari93}.

Note that the solution to this optimization problem (Equation \ref{eq:parametricreduction2}) is not unique, since we can assign different fixed values for non-free parameters in $M$. Intuitively, to determine the appropriate values for non-free parameters, our best choice is the $M$ that intersects at $p_s$. However, in general cases where $p_s$ is NOT specified in advance, it is natural to assign non-free parameters to a neutral value (e.g., zero). This treatment is used by the general CIF (see Section \ref{sec:cif}). If $p_s$ is specified in advance, we can, in principle, further select a $M$ as close to $p_s$ as possible. It turn out that we can develop a sample-specific CIF w.r.t given samples (see Section \ref{sec:cd-cif}).

The rationality of maximally preserving the Fisher information distance can also be interpreted from the  maximum-likelihood (ML) estimation point of view. Let $\hat{\xi}$ be the ML estimators for $\xi$.
The asymptotic normality of ML estimation implies that the distribution of $\hat{\xi}$ is the normal distribution with mean $\xi$ and covariance $\Sigma$, i.e.,
\begin{equation}\label{eq:normalapproximation}
  f(\hat{\xi})\sim \mathcal{N}(\xi, \Sigma) = \frac{1}{Z} exp\{-\frac{1}{2}(\xi-\hat{\xi})^T \Sigma^{-1} (\xi-\hat{\xi})\}
\end{equation}
where the inverse of $\Sigma$ can be asymptotically estimated using the Fisher information matrix $G_\xi$, as suggested by the Cram\'{e}r-Rao bound \cite{rao45attainable} and the asymptotic normality of ML estimation. From the Fisher information distance given in Equation \ref{eq:fisherdistance}, the exponent part of Equation \ref{eq:normalapproximation} is just the opposite of the half squared Fisher information distance between two distributions $p$ and $\hat{p}$ determined by the close parameters $\xi$ and $\hat{\xi}$, respectively. Hence a large Fisher information distance means a lower likelihood. It turns out that, in density estimates, maximally preserving the expected Fisher information distance after the projection $\Gamma_M$ (Equation \ref{eq:parametricreduction2}) is equivalent to maximally preserving the likelihood-structure among close distributions.
In supervised learning (e.g., classification), maximally preserving FID can also effectively preserve the likelihood-structure among different class densities (the underlying distributions of classes), which is beneficial against sample noises.
Recall that sample noises always reduce the FID among class densities in a statistical sense, which lead to the reduced discrimination marginality between two class densities. Hence, for noisy data, the model that maximally preserving FID can capture the dominant discrimination between class densities.

To solve the optimization problem in Equation \ref{eq:parametricreduction2}, we propose a parameter reduction criterion called the \emph{Confident-Information-First} (CIF) principle, described as follows.
The Fisher information distance $D(p_t, p_s)$ can be decomposed into the distances of two orthogonal parts \cite{Amari93}.
Moreover, it is possible to divide the system parameters in $S$ into two categories (corresponding to the two decomposed distances), \textit{i.e.}, the parameters with ``major'' variations and the parameters with ``minor'' variations, according to their contributions to the whole information distance. The former refers to parameters that are important for reliably distinguishing the true distribution from the sampling distribution, thus considered as ``confident". On the other hand, the parameters with minor contributions can be considered as less reliable.
Hence, the CIF principle can be stated as parametric reduction that preserves the confident parameters and rules out less confident parameters. We will theoretically show that CIF leads to an optimal submanifold $M$ in terms of the optimization problem defined in Equation \ref{eq:parametricreduction2}.
It is worth emphasizing that the proposed CIF as a principle of parametric reduction is fundamentally different from the traditional feature reduction (or feature extraction) methods \cite{Fodor02asurvey,lee07nonlinear}. The latter focus on directly reducing the dimensionality on feature space by retaining maximal variations in the data, e.g., Principle Components Analysis (PCA) \cite{abdi10pcareview}, while CIF offers a principled method to deal with high-dimensional data in the parameter spaces by a strategy that is derived from the first principle \footnote{The Fisher-Rao metric is considered as the first principle to measure the distance between distributions since it is the unique metric meeting a set of natural axioms for the distribution metric, as stated earlier.}, independent of the scales of features.

The main contributions of this paper are:
\begin{enumerate}
  \item We incorporate the Fisher information distance into the modelling of the intrinsic variations in the data that give rise to the desired model in the framework of IG.
  \item We propose a CIF principle for parametric reduction to maximally preserve the confident parameters and ruling out less confident ones.
  \item For binary multivariate distributions, we theoretically show that CIF could analytically lead to an optimal submanifold w.r.t. the parametric reduction problem in Equation \ref{eq:parametricreduction2}.
  \item The utility of CIF, i.e., the derivation of probabilistic models, is illustrated by revisiting the Boltzmann machines (BM). We show by examples that some existing probabilistic models, e.g., the fully visible BM (VBM) and the BM with hidden units, comply with the CIF principle and can be derived from it.
  \item Given certain samples, we propose a sample-specific CIF-based model selection scheme for the Bolzmann machines. It leads to a significant improvement in a series of density estimation experiments.
\end{enumerate}


\section{Theoretical Foundations of IG}\label{sec:theoryIG}
In this section, we introduce and develop the theoretical foundations of IG\cite{Amari93} for the manifold $S$ of binary multivariate distributions with a given number of variables $n$, i.e., the open simplex of all probability distributions over binary vector $x \in \{0,1\}^{n}$. This will lay the foundation for our theoretical deviation of the \emph{CIF}.
\subsection{Notations for Manifold S}\label{sec:def}
In IG, a family of probability distributions is considered as a differentiable manifold with certain parametric coordinate systems. In the case of binary multivariate distributions, four basic coordinate systems are often used \cite{Amari93}\cite{hou2013}: $p$-coordinates, $\eta$-coordinates, $\theta$-coordinates and the mixed $\zeta$-coordinates. The $\zeta$-coordinates is of vital importance for our analysis.

For the $p$-coordinates $[p]$, the probability distribution over $2^n$ states of $x$ can be completely specified by any $2^n-1$ positive numbers indicating the probability of the corresponding exclusive states on $n$ binary variables. For example, the $p$-coordinates of $n=2$ variables could be $[p]=(p_{01}, p_{10}, p_{11})$. Note that IG requires all probability terms to be positive \cite{Amari93}.
For simplicity, we use the capital letters $I,J,\dots$ to index the coordinate parameters of probabilistic distribution. An index $I$ can be regarded as a subset of $\{1,2,\dots,n\}$. Additionally, $p_I$ stands for the probability that all variables indicated by $I$ equal to one and the complemented variables are zero. For example, if $I=\{1,2,4\}$ and $n=4$, we have:
$$p_I=p_{1101}=Prob(x_1=1,x_2=1,x_3=0,x_4=1)$$ Note that the null set can also be a legal index of the $p$-coordinates, which indicates the probability that all variables are zero, denoted as $p_{0 \dots 0}$.

The $\eta$-coordinates $[\eta]$ are defined by:
\begin{equation}\label{eq:etacoordinate}
 \eta_I=E[X_I]=Prob\{\prod_{i\in I}x_i = 1\}
\end{equation}
where the value of $X_I$ is given by $\prod_{i\in I}x_i$ and the expectation is taken with respect to the probability distribution over $x$. Grouping the coordinates by their orders, the $\eta$-coordinates are denoted as $[\eta]=(\eta^1_i, \eta^2_{ij},\dots, \eta^n_{1,2...n})$, where the superscript indicates the order number of the corresponding parameter. For example, $\eta^2_{ij}$ denotes the set of all $\eta$ parameters with the order number two.

The $\theta$-coordinates (natural coordinates) $[\theta]$ are defined by:
\begin{equation}\label{eq:thetacoordinate}
 \log{p(x)}=\sum_{I\subseteq\{1,2,\dots,n\}, I\neq NullSet}{\theta^I X_I} - \psi(\theta)
\end{equation}
where $\psi(\theta)=\log(\sum_x{exp\{\sum_I{\theta^I X_I(x)}\}})$ is the cumulant generating function and its value equals to $-\log Prob\{x_i=0, \forall i\in \{1,2,...,n\}\}$. By solving the linear system \ref{eq:thetacoordinate}, we have $\theta^I = \sum_{K \subseteq I} (-1)^{|I-K|} log(p_{K})$. The $\theta$-coordinate is denoted as $[\theta]=(\theta^{i}_1, \theta^{ij}_2,\dots, \theta^{1,...,n}_n)$, where the subscript indicates the order number of the corresponding parameter.
Note that the order indices locate at different positions in $[\eta]$ and $[\theta]$ following the convention in \cite{amari92igbm}.

The relation between coordinate systems $[\eta]$ and $[\theta]$ is bijective.
More formally, they are connected by the Legendre transformation:
\begin{equation}\label{eq:trans_eta_theta}
 \theta^I=\frac{\partial \phi(\eta)}{\partial \eta_I}, \eta_I=\frac{\partial \psi(\theta)}{\partial \theta^I}
\end{equation}
where $\psi(\theta)$ is given in Equation \ref{eq:thetacoordinate} and $\phi(\eta)=\sum_x p(x;\eta) \log p(x;\eta)$ is the negative of entropy. It can be shown that $\psi(\theta)$ and $\phi(\eta)$ meet the following identity \cite{Amari93}:
\begin{equation}\label{eq:legedre}
 \psi(\theta)+\phi(\eta)-\sum{\theta^I\eta_I}=0
\end{equation}

The $l$-mixed $\zeta$-coordinates $[\zeta]_l$ are defined by:
\begin{equation}\label{eq:lmixedcoordinates}
\!\!\!\![\zeta]_l\!\!=\!\![\eta^{l-}, \theta_{l+}]\!\!=\!\!(\eta^1_i, \eta^2_{ij},\dots, \eta^l_{i,j,\dots,k}, \theta^{i,j,\dots,k}_{l+1},\dots, \theta^{1,...,n}_n)
\end{equation}
where the first part consists of $\eta$-coordinates with order less or equal to $l$ and the second part consists of $\theta$-coordinates with order greater than $l$, $l \in \{1,...,n-1\}$.

\subsection{Fisher Information Matrix for Parametric Coordinates} \label{sec:fisherinformationmatrix}
For a general coordinate system $[\xi]$, the $i$-th row and $j$-th column element of the Fisher information matrix for $[\xi]$ (denoted by $G_{\xi}$) is defined as the covariance of the scores of $[\xi_i]$ and $[\xi_j]$ \cite{rao45attainable}:
 $$g_{ij}=E[\frac{\partial \log p(x;\xi)}{\partial \xi_i}\cdot \frac{\partial \log p(x;\xi)}{\partial \xi_j}]$$
under the regularity condition that the partial derivatives exist.
The Fisher information measures the amount of information in the data that a statistic carries about the\linebreak unknown parameters \cite{kass89asymptotic}.
The Fisher information matrix is of vital importance to our analysis, because the inverse of Fisher information matrix gives an asymptotically tight lower bound to the covariance matrix of any unbiased estimate for the considered parameters \cite{rao45attainable}. 
Another important concept related to our analysis is the orthogonality defined by Fisher information. Two coordinate parameters $\xi_i$ and $\xi_j$ are called orthogonal if and only if their Fisher information vanishes, \textit{i.e}., $g_{ij}=0$, meaning that their influences on the log likelihood function are uncorrelated.

The Fisher information for $[\theta]$ can be rewritten as $g_{IJ}=\frac{\partial^2 \psi(\theta)}{\partial \theta ^I \partial \theta ^J}$, and for $[\eta]$, it is $g^{IJ}=\frac{\partial^2 \phi(\eta)}{\partial \eta_I \partial \eta_J}$
\cite{Amari93}. Let $G_\theta=(g_{IJ})$ and $G_\eta =(g^{IJ})$ be the Fisher information matrices for $[\theta]$ and $[\eta]$, respectively. It can be shown that $G_\theta$ and $G_\eta$ are mutually inverse matrices, \textit{i.e}., $\sum_J{g^{IJ}g_{JK}}=\delta^I_K$, where $\delta^I_K=1$ if $I=K$ and zero otherwise \cite{Amari93}.
In order to generally compute $G_\theta$ and $G_\eta$, we develop the following Propositions \ref{prop:fishermatrix} \linebreak and \ref{prop:fishermatrix_eta}. Note that Proposition \ref{prop:fishermatrix} is a generalization of Theorem 2 in \cite{amari92igbm}.
\begin{proposition}\label{prop:fishermatrix}
The Fisher information between two parameters $\theta^I$ and $\theta^J$ in $[\theta]$, is given by:
\begin{equation}\label{eq:proposiationFisherMetrictheta}
 g_{IJ}(\theta)=\eta_{I\cup J}-\eta_I \eta_J
\end{equation}
\end{proposition}
\begin{proof}
in Appendix \ref{appendix:thetafisher}.
\end{proof}

\begin{proposition}\label{prop:fishermatrix_eta}
The Fisher information between two parameters $\eta_I$ and $\eta_J$ in $[\eta]$, is given by:
\begin{equation}\label{eq:proposiationFisherMetriceta}
 g^{IJ}(\eta)=\sum_{K\subseteq I\cap J}{(-1)^{|I-K|+|J-K|} \cdot \frac{1}{p_{K}}}
\end{equation}
where $|\cdot|$ denotes the cardinality operator.
\end{proposition}
\begin{proof}
in Appendix \ref{appendix:etafisher}. 
\end{proof}

We take the probability distribution with three variables for example. Based on Equation \ref{eq:proposiationFisherMetriceta}, the Fisher information between $\eta_I$ and $\eta_J$ can be calculated, e.g., $g^{IJ}=\frac{1}{p_{000}}+\frac{1}{p_{010}}$ if $I=\{1,2\}$ and $J=\{2,3\}$, $g^{IJ}=-(\frac{1}{p_{000}}+\frac{1}{p_{010}}+\frac{1}{p_{100}}+\frac{1}{p_{110}})$ if $I=\{1,2\}$ and $J=\{1,2,3\}$, and etc.

Based on $G_\eta$ and $G_\theta$, we can calculate the Fisher information matrix $G_\zeta$ for the $[\zeta]_l$.
\begin{proposition}\label{prop:fishermatrix_mix}
The Fisher information matrix $G_\zeta$ of $[\zeta]_l$ is given by:
 \begin{equation}\label{eq:estimationerror}
 G_{\zeta}= \left(
    \begin{array}{cc}
     A & 0 \\
     0 & B \\
    \end{array}
    \right)
 \end{equation}
where $A=((G_\eta^{-1})_{I_\eta})^{-1}$, $B=((G_\theta^{-1})_{J_\theta})^{-1}$, $G_\eta$ and $G_\theta$ are the Fisher information matrices of $[\eta]$ and $[\theta]$, respectively, $I_\eta$ is the index set of the parameters shared by $[\eta]$ and $[\zeta]_l$, \textit{i.e}., $\{\eta^1_i,..., \eta^l_{i,j,...,k}\}$, and $J_\theta$ is the index set of the parameters shared by $[\theta]$ and $[\zeta]_l$, \textit{i.e}., $\{\theta^{i,j,...,k}_{l+1},\dots, \theta^{1,...,n}_n\}$.
\end{proposition}
\begin{proof}
in Appendix \ref{appendix:mixfisher}.
\end{proof}

\section{The General CIF Principle}\label{sec:cif}
The general manifold $S$ of all probability distributions over binary vector $x \in \{0,1\}^{n}$ could be exactly represented using the $2^{n}-1$ parametric coordinates. Given a target distribution $q(x)\in S$, we consider the problem of realizing it by a lower-dimensionality submanifold $M$. This is defined as the problem of parametric reduction for multivariate binary distributions.

In this section, we will formally illuminate the general CIF for parametric reduction. Intuitively, if we can construct a coordinate system so that the confidences of its parameters entail a natural hierarchy, in which high confident parameters are significantly distinguished from and orthogonal to lowly confident ones, then we can conveniently implement CIF by keeping the high confident parameters unchanged and setting the lowly confident parameters to neutral values. As described in Section \ref{sec:intro}, the confidence of parameters should be assessed according to their contributions to the expected information distance. Therefore, the choice of coordinates in CIF is crucial to its usage. This strategy is infeasible in terms of $p$-coordinates, $\eta$-coordinates or $\theta$-coordinates, since the orthogonality condition cannot hold in these coordinate systems. In this section, we will show that the $l$-mixed-coordinates $[\zeta]_l$ meets the requirement of CIF.

To grasp an intuitive picture for the general CIF strategy and its significance w.r.t mixed-coordinates $[\zeta]_l$, we will first show that the $l$-mixed-coordinates $[\zeta]_l$ meets the requirement of CIF in typical distributions that generate real-world datasets. Then we will prove that CIF could lead to an optimal submanifold w.r.t. the parametric reduction problem in Equation \ref{eq:parametricreduction2}, in general cases.

\subsection{The CIF in Typical Distributions}\label{sec:ciftypical}
To facilitate our analysis, we make a basic assumption on the underlying distributions $q(x)$ that at least $(2^n-2^{n/2})$ $p$-coordinates are of the scale $\epsilon$, where $\epsilon$ is a sufficiently small value. Thus, residual $p$-coordinates (at most $2^{n/2}$) are all significantly larger than zero (of scale $\Theta(1/2^{(n/2)})$), and their sum approximates one. Note that these assumptions are common situations in real-world data collections \cite{highdimensionalstatistics}, since the frequent (or meaningful) patterns are only a small fraction of all of the\linebreak system states.

Next, we introduce a small perturbation $\Delta p$ to the $p$-coordinates $[p]$ for the true distribution $q(x)$. The perturbed distribution is denoted as $q'(x)$. For $p$-coordinates that are significantly larger than zero, the scale of each fluctuation $\Delta p_I$ is assumed to be proportional to the standard variation of corresponding $p$-coordinate $p_I$ by some small coefficients (upper bounded by a constant $a$), which can be approximated by the inverse of the square root of its Fisher information via the Cram\'{e}r--Rao bound. It turns out that we can assume the perturbation $\Delta p_I$ to be $a\sqrt{p_I}$. For $p$-coordinates with a small value (approximates zero), the scale of each fluctuation $\Delta p_I$ is assumed to be proportional to $a p_I$.

In this section, we adopt the $l$-mixed-coordinates $[\zeta]_l = (\eta^{l-};\theta_{l+})$, where $l=2$ is used in the following analysis. Let $\Delta \zeta_q = (\Delta \eta^{2-};\Delta \theta_{2+})$ be the incremental of mixed-coordinates after the perturbation. The squared Fisher information distance $D^2(q,q')=(\Delta \zeta_q)^T G_\zeta \Delta \zeta_q$ could be decomposed into the direction of each coordinate in $[\zeta]_l$. We will clarify that, under typical cases, the scale of the Fisher information distance in each coordinate of $\theta_{l+}$ (will be reduced by CIF) is asymptotically negligible, compared to that in each coordinate of $\eta^{l-}$ (will be preserved by CIF).

The scale of squared Fisher information distance in the direction of $\eta_{I}$ is proportional to $\Delta \eta_I \cdot (G_\zeta)_{I,I} \cdot \Delta \eta_I$, where $(G_\zeta)_{I,I}$ is the Fisher information of $\eta_{I}$ in terms of the mixed-coordinates $[\zeta]_2$.
From Equation~\ref{eq:etacoordinate}, for any $I$ of order one (or two), $\eta_{I} $ is the sum of $2^{n-1}$ (or $2^{n-2}$) $p$-coordinates, and the scale is $\Theta(1)$. Hence, the incremental $\Delta \eta^{2-}$ is proportional to $\Theta(1)$, denoted as $a\cdot \Theta(1)$.
It is difficult to give an explicit expression of $(G_\zeta)_{I,I}$ analytically. However, the Fisher information $(G_\zeta)_{I,I}$ of $\eta_{I}$ is bounded by the $(I,I)$-th element of the inverse covariance matrix \cite{bobrovsky87}, which is exactly $1/g^{I,I}(\theta)=\frac{1}{\eta_I - \eta_I^2}$ (see Proposition~\ref{prop:fishermatrix_mix}). Hence, the scale of $(G_\zeta)_{I,I}$ is also $\Theta(1)$. It turns out that the scale of squared Fisher information distance in the direction of $\eta_{I}$ is $a^2 \cdot \Theta(1)$.

Similarly, for the part $\theta_{2+}$, the scale of squared Fisher information distance in the direction of $\theta^{J}$ is proportional to $\Delta \theta^J \cdot (G_\zeta)_{J,J} \cdot \Delta \theta^J$, where $(G_\zeta)_{J,J}$ is the Fisher information of $\theta^{J}$ in terms of the mixed-coordinates $[\zeta]_2$.
The scale of $\theta^{J}$ is maximally $f(k)|log(\sqrt{\epsilon})|$ based on Equation~\ref{eq:thetacoordinate}, where $k$ is the order of $\theta^{J}$ and $f(k)$ is the number of $p$-coordinates of scale $\Theta(1/2^{(n/2)})$ that are involved in the calculation of $\theta^{J}$. Since we assume that $f(k)\le 2^{(n/2)}$, the maximum scale of $\theta^{J}$ is $2^{(n/2)}|log(\sqrt{\epsilon})|$. Thus, the incremental $\Delta \theta^J$ is of a scale bounded by $a \cdot 2^{(n/2)}|log(\sqrt{\epsilon})|$.
Similar to our previous deviation, the Fisher information $(G_\zeta)_{J,J}$ of $\theta^{J}$ is bounded by the $(J,J)$-th element of the inverse covariance matrix, which is exactly $1/g_{J,J}(\eta)$ (see Proposition~\ref{prop:fishermatrix_mix}). Hence, the scale of $(G_\zeta)_{J,J}$ is $(2^k-f(k))^{-1}\epsilon$. In summary, the scale of squared Fisher information distance in the direction of $\theta^{J}$ is bounded by the scale of $a^2 \cdot \Theta(2^n \epsilon \frac{|log(\sqrt{\epsilon})|^2}{2^k-f(k)})$. Since $\epsilon$ is a sufficiently small value and $a$ is constant, the scale of squared Fisher information distance in the direction of $\theta^{J}$ is asymptotically zero.

According to our above analysis, the confidences of coordinate parameters (measured by the decomposed Fisher information distance) in $[\zeta]_l$ entail a natural hierarchy: the first part of high confident parameters $[\eta^{l-}]$ are significantly larger than the second part of low confident parameters $[\theta_{l+}]$. Additionally, those low confident parameters $[\theta_{l+}]$ have the neutral value of zero.
Moreover, the parameters in $[\eta^{l-}]$ are orthogonal to the ones in $[\theta_{l+}]$, indicating that we could estimate these two parts independently \cite{hou2013}. Hence, we can implement the CIF for parametric reduction in $[\zeta]_l$ by replacing low confident parameters with neutral value zero and reconstructing the resulting distribution. It turns out that the submanifold of $S$ tailored by CIF becomes $[\zeta]_{l_t}=(\eta_{i}^1,..., \eta_{ij...k}^l,0,\dots,0)$. We call $[\zeta]_{l_t}$ the $l$-tailored-mixed-coordinates.

To verify our theoretical analysis, we conduct a simulation on the ratio of FID that is preserved by the $l$-tailored-mixed-coordinates ($l=2$) $[\zeta]_{l_t}$ w.r.t. the original mixed-coordinates $[\zeta]$. First we randomly select real distribution $p_t$ with $n$ variables, where the distribution satisfies the basic assumption that we make in the beginning of this section (the $2^{n/2}$ significant $p$-coordinates are generated based on Jeffery prior, left $p$-coordinates are set to a small constant). Then we generate the sample distribution $p_s$ based on random samples drawn from the real distribution. Last, we calculate the FID between $p_t$ and $p_s$ in terms of the $[\zeta]$ and $[\zeta]_{l_t}$ respectively. The result is shown in Table \ref{tab:FIDpreservesimulation}. We can see that the $[\zeta]_{l_t}$ can indeed preserve most of the FID, which is consistent with our theoretical analysis.
\begin{table}[!t]
\caption{Simulation on the FID preserved by $[\zeta]_{l_t}$ ($l=2$)}
\label{tab:FIDpreservesimulation}
\centering
\begin{tabular}{|c||c||c|c|}
\hline
\multirow{2}{*}{$n$} & ratio of preserved & \multicolumn{2}{c|}{ratio of preserved FID}\\
\cline{3-4}
& parameters & mean & standard deviation \\
\hline
3 & 0.857 & 0.9972 & 0.0055 \\
\hline
4 & 0.667 & 0.9963 & 0.0043 \\
\hline
5 & 0.484 & 0.9923 & 0.0054 \\
\hline
6 & 0.333 & 0.9824 & 0.0112 \\
\hline
7 & 0.220 & 0.9715 & 0.0111 \\
\hline
\end{tabular}
\end{table}

\begin{figure}
 \centering
 \includegraphics[width=0.4\textwidth]{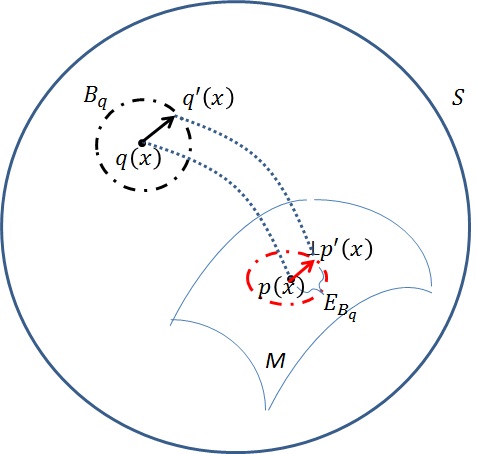}
 \caption{By projecting a point $q(x)$ on $S$ to a submanifold $M$, the $l$-tailored mixed-coordinates $[\zeta]_{l_t}$ gives a desirable $M$ that maximally preserves the expected Fisher information distance when projecting a $\varepsilon$-neighborhood centered at $q(x)$ onto $M$.}
 \label{fig:expectedFD} 
\end{figure}

\subsection{The CIF Leads to an Optimal Submanifold $M$}
Let $B_q$ be a $\varepsilon$-sphere surface centered at $q(x)$ on manifold $S$, \textit{i.e}., $B_q= \{q' \in S |\| KL(q,q') = \varepsilon\}$, where $KL(\cdot,\cdot)$ denotes the KL divergence and $\varepsilon$ is small. Additionally, $q'(x)$ is a neighbor of $q(x)$ uniformly sampled on $B_q$, as illustrated in Figure~\ref{fig:expectedFD}. Recall that, for a small $\varepsilon$, the KL divergence can be approximated by half of the squared Fisher information distance. Thus, in the parameterization of $[\zeta]_l$, $B_q$ is indeed the surface of a hyper-ellipsoid (centered at $q(x)$) determined by $G_\zeta$. The following proposition shows that the general CIF would lead to an optimal submanifold $M$ that maximally preserves the expected information distance, where the expectation is taken upon the uniform neighborhood, $B_q$.

\begin{proposition}\label{prop:GeometricView}
Consider the manifold $S$ in $l$-mixed-coordinates $[\zeta]_l$. Let $k$ be the number of free parameters in the $l$-tailored-mixed-coordinates $[\zeta]_{l_t}$. Then, among all $k$-dimensional submanifolds of $S$, the submanifold determined by $[\zeta]_{l_t}$ can maximally preserve the expected information distance
induced by the Fisher--Rao metric.
\end{proposition}
\begin{proof}
in Appendix \ref{appendix:geometriccif}.
\end{proof}

\section{CIF-based Interpretation of Boltzmann Machine}\label{sec:implementationCIF}
In previous section, a general CIF is uncovered in the $[\zeta]_l$ coordinates for multivariate binary distributions. Now we consider the implementations of CIF when $l$ equals to 2 using the Boltzmann machines (BM).

\subsection{Introduction to the Boltzmann Machines}
In general, a BM \cite{ackley85BM} is defined as a stochastic neural network consisting of visible units $x\in\{0,1\}^{n_x}$ and hidden units $h\in\{0,1\}^{n_h}$, where each unit fires stochastically depending on the weighted sum of its inputs. The energy function is defined as follows:
\begin{equation}\label{eq:energyBM}
   E_{BM}(x,h;\xi)=-\frac{1}{2}x^TUx-\frac{1}{2}h^TVh -x^TWh -b^Tx-d^Th
\end{equation}
where $\xi=\{U,V,W,b,d\}$ are the parameters: visible-visible interactions ($U$), hidden-hidden interactions ($V$), visible-hidden interactions ($W$), visible self-connections ($b$) and hidden self-connections ($d$). The diagonals of $U$ and $V$ are set to zero.
We can express the Boltzmann distribution over the joint space of $x$ and $h$ as below:
\begin{equation}\label{eq:probBM}
    p(x,h;\xi)=\frac{1}{Z}exp\{-E_{BM}(x,h;\xi)\}
\end{equation}
where $Z$ is a normalization factor.
\subsubsection{The Coordinates for Boltzmann Machines}
~\\
\vspace{-4mm}
\par
Let $B$ be the set of Boltzmann distributions realized by BM. Actually, $B$ is a submanifold of the general manifold $S_{xh}$ over $\{x,h\}$. From Equation (\ref{eq:probBM}) and (\ref{eq:energyBM}), we can see that $\xi=\{U,V,W,b,d\}$ plays the role of $B$'s coordinates in $\theta$-coordinates (Equation \ref{eq:thetacoordinate}) as follows:
\begin{eqnarray}\label{eq:bmtotheta}
  \theta_1 &:& \theta^{x_i}_1 = b_{x_i}, \theta^{h_j}_1 = d_{h_j} (\forall x_i\in x, h_j \in h) \nonumber \\
  \theta_2 &:& \theta^{x_ix_j}_2 = U_{x_i,x_j}, \theta^{x_ih_j}_2 = W_{x_i,h_j}, \nonumber \\
    && \theta^{h_ih_j}_2 = V_{h_i,h_j}, (\forall x_i, x_j\in x; h_i,h_j \in h) \nonumber \\
  \theta_{2+} &:& \theta^{x_i\dots x_j h_u \dots h_v}_m = 0, m>2, \nonumber \\
    &&(\forall x_i,\dots, x_j \in x; h_u,\dots,h_v \in h)
\end{eqnarray}
So the $\theta$-coordinates for BM is given by:
\begin{equation}\label{eq:thetaforBM}
  [\theta]_{BM}=(\underbrace{\theta^{x_i}_1, \theta^{h_j}_1}_{1-order}, \underbrace{\theta^{x_ix_j}_2, \theta^{x_ih_j}_2, \theta^{h_ih_j}_2}_{2-order}, \underbrace{0, \dots, 0 }_{orders>2}).
\end{equation}

The VBM and restricted BM are special cases of the general BM. Since VBM has $n_h=0$ and all the visible units are connected to each other, the parameters of VBM are $\xi_{vbm}=\{U,b\}$ and $\{V,W,d\}$ are all set to zero. For RBM, it has connections only between hidden and visible units. Thus, the parameters of RBM are $\xi_{rbm}=\{W,b,d\}$ and $\{U,V\}$ are set to zero.

\subsubsection{The Gradient-based Learning of BM} \label{sec:learningMLforBM}
~\\
\vspace{-4mm}
\par
~~Given the sample $\underline{x}$ that generated from the underlying distribution, the \emph{maximum-likelihood} (ML) is a commonly used gradient ascent method for training BM in order to maximize the log-likelihood $\log p(\underline{x};\xi)$ of the parameters $\xi$ \cite{hinton05CD}. Based on Equation (\ref{eq:probBM}), the log-likelihood is given as follows:
$$\log p(\underline{x};\xi)=log \sum_h e^{-E(\underline{x},h;\xi)} - log \sum_{x',h'} e^{-E(x',h';\xi)}$$
Differentiating the log-likelihood, the gradient with respect to $\xi$ is as follows:
\begin{eqnarray}\label{eq:gradientloglikelihood}
\!\!\!\!  \frac{\partial \log p(\underline{x};\xi)}{\partial \xi}&=&\sum_h p(h|\underline{x};\xi)\frac{\partial [-E(\underline{x},h;\xi)]}{\partial \xi} \nonumber \\
\!\!\!\! &&- \sum_{x',h'} p(h'|x';\xi)\frac{\partial [-E(x',h';\xi)]}{\partial \xi}
\end{eqnarray}
where $\frac{\partial E(x,h;\xi)}{\partial \xi}$ can be easily calculated from Equation (\ref{eq:energyBM}). Then we can obtain the stochastic gradient using Gibbs sampling \cite{gilk96mcmc} in two phases: sample $\underline{h}$ given $\underline{x}$ for the first term, called the positive phase, and sample $(\underline{x}',\underline{h}')$ from the stationary distribution $p(x',h';\xi)$ for the second term, called the negative phase. Now with the resulting stochastic gradient estimation, the learning rule is to adjust $\xi$ by:
\begin{eqnarray}\label{eq:learnrulestochasticgradient}
  \Delta \xi_{ml} &=& \varepsilon \frac{\partial \log p(\underline{x};\xi)}{\partial \xi} \nonumber \\
   &\propto& - \langle \frac{\partial E(\underline{x},\underline{h};\xi)}{\partial \xi} \rangle_0 + \langle \frac{\partial E(\underline{x'},\underline{h'};\xi)}{\partial \xi} \rangle_\infty
\end{eqnarray}
where $\varepsilon$ is the learning rate, $\langle\cdot \rangle_0$ denotes the average using the sample data and $\langle\cdot \rangle_\infty$ denotes the average with respect to the stationary distribution $p(x,h;\xi)$ after the corresponding Gibbs sampling phases.

To avoid the difficulty of computing the log-likelihood gradient, the Contrastive divergence (CD) \cite{hinton05CD} realizes the gradient descent of a different objective function, shown as follows:
\begin{eqnarray}\label{CD-learning}
 \Delta \xi_{cd} &=&-\varepsilon \frac{\partial (KL(p_0||p) - KL(p_m||p))}{\partial \xi} \nonumber \\
    &\propto&  - \langle \frac{\partial E(\underline{x},\underline{h};\xi)}{\partial \xi} \rangle_0 + \langle \frac{\partial E(\underline{x'},\underline{h'};\xi)}{\partial \xi} \rangle_m
\end{eqnarray}
where $p_0$ is the sample distribution, $p_m$ is the distribution by starting the Markov chain with the data and running $m$ steps and $KL(\cdot||\cdot)$ denotes the KL divergence. CD can be seen as an approximation to ML by replacing the last expectation $\langle\cdot \rangle_\infty$ with $\langle\cdot \rangle_m$.


\subsection{The Fully Visible Boltzmann Machine}\label{sec:geometryBM}

Consider the parametric reduction on the manifold $S$ over $\{x\}$ and end up with a $k$-dimensional submanifold $M$ of $S$, where $k \ll2^{n_x}\!\!-\!\!1$ is the number of free parameters in $M$.
$M$ is set to be the same dimensionality as VBM, i.e., $k=\frac{n_x(n_x+1)}{2}$, so that all candidate submanifolds are comparable to the submanifold $M_{vbm}$ endowed by VBM.
Next, the rationale underlying the design of $M_{vbm}$ can be illuminated using the general CIF.
\subsubsection{The Derivation of VBM via CIF}
~\\
\vspace{-4mm}
\par
In the following corollary, we will show that the statistical manifold $M_{vbm}$ is the optimal parameter subspace spanned by those directions with high confidences in terms of CIF.
\begin{corollary}\label{prop:SBMCIF}
Given the general manifold $S$ in $2$-mixed-coordinates $[\zeta]_2$, VBM (with coordinates $[\zeta]_{2_t}$) defines an $k$-dimensional submanifold of $S$ that can maximally preserve the expected Fisher information distance induced by Fisher-Rao metric.
\end{corollary}
\begin{proof}
in Appendix \ref{appendix:SBMCIF}.
\end{proof}
\subsubsection{The Interpretation of VBM Learning via CIF}
~\\
\vspace{-4mm}
\par
To learn such $[\zeta]_{2_t}$, we need to learn the parameters $\xi$ of VBM such that its stationary distribution preserves the same coordinates $[\eta^{2-}]$ as target distribution $q(x)$. Actually, this is exactly what traditional gradient-based learning algorithms intend to do. Next proposition shows that the ML learning of VBM is equivalent to learn the coordinates $[\zeta]_{2_t}$.

\begin{proposition}\label{prop:sbmmlcloseform}
    Given the target distribution $q(x)$ with 2-mixed coordinates: $$[\zeta]_2=(\eta^1_i, \eta^2_{ij},\theta_{2+})$$
    the coordinates of the stationary distribution of VBM trained by ML are uniquely given by: $$[\zeta]_{2_t}=(\eta_{i}^1, \eta_{ij}^2,\theta_{2+}=0)$$
\end{proposition}
\begin{proof}
in Appendix \ref{appendix:sbmmlcloseform}.
\end{proof}

\subsection{The Boltzmann Machine with Hidden Units}\label{subsec:geometryRBM}
In previous section, the CIF is applied to models without hidden units and leads to VBM by preserving the $1$-order and $2$-order $\eta$-coordinates. In this section, we will investigate the cases where hidden units are introduced.

Let $S_{xh}$ be the manifold of distributions over the joint space of visible units $x$ and hidden units $h$. A general BM produces a stationary distribution $p(x,h;\xi)\in S_{xh}$ over $\{x,h\}$. Let $B$ denote the submanifold of $S_{xh}$ with probability distributions $p(x,h;\xi)$ realizable by BM.

Given any target distribution $q(x)$, only the marginal distribution of BM over the visible units are specified, leaving the distributions on hidden units vary freely. Let $H_q$ be the submanifold of $S_{xh}$ with probability distributions $q(x,h)$ that have the same marginal distribution as $q(x)$ and the conditional distribution $q(h|x)$ of hidden units is realised by the BM's activation functions with some parameter $\xi_{bm}$.

Then, the best BM is the one that minimizes the distance between $B$ and $H_q$. Due to the existence of hidden units, the solution may not be unique.
In this section, the training process of BM is analysed in terms of manifold projection (described in Section \ref{sec:intro}), following the framework of the learning rule proposed in \cite{amari92igbm}. And we will show that the invariance in the learning of BM is the CIF.

\subsubsection{The Iterative Projection Learning for BM}
~\\
\vspace{-4mm}
\par
The learning algorithm using iterative manifold projection is first proposed in \cite{amari92igbm} and theoretically compared to EM (Expectation and Maximization) algorithm in \cite{amari1995information}.
The learning of RBM can be implemented by the following iterative projection process:
Let $\xi_p^0$ be the initial parameters of BM and $p_0(x,h;\xi_p^0)$ be the corresponding stationary distribution.

For $i=0,1,2,\dots$,
\begin{enumerate}
  \item Put $q_{i+1}(x,h)=\Gamma_H(p_i(x,h;\xi_p^i))$
  \item Put $p_{i+1}(x,h;\xi_p^{i+1})=\Gamma_B(q_{i+1}(x,h))$
\end{enumerate}
where $\Gamma_H(p)$ denotes the projection of $p(x,h;\xi_p)$ to $H_q$, and $\Gamma_B(q)$ denotes the projection of $q(x,h)$ to $B$.
The iteration ends when we reach the fixed points of the projections $p^*$ and $q^*$, that is $\Gamma_H(p^*)=q^*$ and $\Gamma_B(q^*)=p^*$. The iterative projection process is illustrated in Figure \ref{fig:iterativelearning}. The convergence property of this iterative algorithm is guaranteed using the following proposition:
\begin{proposition}\label{prop:monotonicdivergence}
    The monotonic relation holds in the iterative learning algorithm:
    \begin{equation}\label{eq:propmonotonic}
      D[q_{i+1},p_i] \geq D[q_{i+1},p_{i+1}] \geq D[q_{i+2},p_{i+1}]
    \end{equation}
    where the equality holds only for the fixed points of the projections.
\end{proposition}
\begin{proof}
in Appendix \ref{appendix:monotonicdivergence}.
\end{proof}

\begin{figure}
  \centering
  \includegraphics[width=0.4\textwidth]{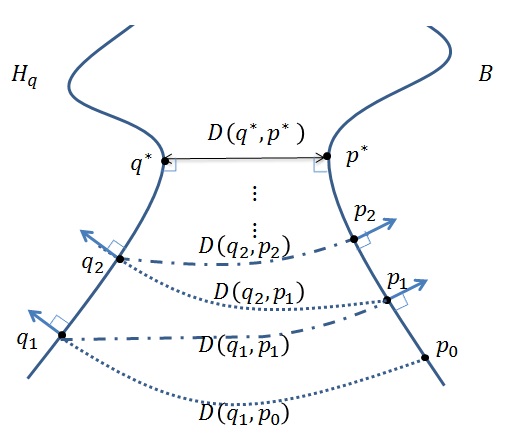}
  \caption{The iterative learning for BM: in searching for the minimum distance between $H_q$ and $B$, we first choose an initial BM $p_0$ and then perform projections $\Gamma_H(p)$ and $\Gamma_B(q)$ iteratively, until the fixed points of the projections $p^*$ and $q^*$ are reached. With different initializations, the iterative projection algorithm may end up with different local minima on $H_q$ and $B$, respectively.}
  \label{fig:iterativelearning} 
\end{figure}

Next two propositions show how the projection $\Gamma_H(p)$ and $\Gamma_B(q)$ are obtained.
\begin{proposition}\label{prop:hqtorbm}
    Given a distribution $p(x,h;\xi_p)\in B$, the projection $\Gamma_H(p)\in H_q$ that gives the minimum divergence $D(H_q, p(x,h;\xi_p))$ from $H_q$ to $p(x,h;\xi_p)$ is the $q(x,h;\xi_{bm}) \in H_q$ that satisfies $\xi_{bm}=\xi_p$.
\end{proposition}
\begin{proof}
in Appendix \ref{appendix:hptorbm}.
\end{proof}

\begin{proposition}\label{prop:projectionRBMcloseform}
    Given $q(x,h;\xi_q) \in H_q$ with mixed coordinates: $[\zeta^{xh}]_{q}=(\eta_{x_i}^1, \eta_{h_j}^1, \eta_{x_ix_j}^2, \eta_{x_ih_j}^2, \eta_{h_ih_j}^2,\theta_{2+})$,
    the coordinates of the learnt projection $\Gamma_B(q) \in B$ are uniquely given by the tailored mixed coordinates:
\begin{equation}\label{eq:mixedcoordinatenewProjectionRBM}
  [\zeta^{xh}]_{\Gamma_B(q)}=(\eta_{x_i}^1, \eta_{h_j}^1, \eta_{x_ix_j}^2, \eta_{x_ih_j}^2, \eta_{h_ih_j}^2, \theta_{2+}=0)
\end{equation}
\end{proposition}
\begin{proof}
This proof comes in three parts:
\begin{enumerate}
  \item the projection $\Gamma_B(q)$ of $q(x,h)$ on $B$ is unique;
  \item this unique projection $\Gamma_B(q)$ can be achieved by minimizing the divergence $D[q(x,h),B]$ using gradient descent method;
  \item The mixed coordinates of $\Gamma_B(q)$ is exactly the one given in Equation (\ref{eq:mixedcoordinatenewProjectionRBM}).
\end{enumerate}
See Appendix \ref{appendix:projectionRBMcloseform} for the detailed proof.
\end{proof}

\subsubsection{The Interpretation for BM Learning via CIF}
~\\
\vspace{-4mm}
\par
The iterative projection learning (IP) gives us an alternative way to investigate the learning process of BM.
Based on the CIF principle in Section \ref{sec:cif}, we can see that the process of the projection $\Gamma_{B}(q_i)$ can be derived from CIF, i.e., highly confident coordinates $[\eta_{x_i}^1, \eta_{h_j}^1, \eta_{x_ix_j}^2, \eta_{x_ih_j}^2, \eta_{h_ih_j}^2]$ of $q_i$ are preserved while lowly confident coordinates $[\theta_{2+}]$ are set to neutral value zero, given in Equation \ref{eq:mixedcoordinatenewProjectionRBM}.

In summary, the essential parts of the real distribution that can be learnt by BM (with and without hidden units) are exactly the confident coordinates indicated by the CIF principle.

\section{Experimental Study on Sample-specific CIF} \label{sec:cd-cif}
In this section, we will empirically investigate the sample-specific CIF principle in density estimation tasks for Boltzmann machines. More specifically, we aim to adaptively determine free parameters in BM, such that BM can trained be as close to the sample distribution as possible w.r.t. specific samples. For VBM, we will investigate how to use CIF to modify the topology of VBM by reducing less confident connections among visible units with respect to given samples. For BM with hidden units, we extend the traditional restricted BM (RBM) by allowing connections among visible units, called vRBM. Then we apply CIF on vRBM to emphasis the learning on confident connections among visible units.
\subsection{The Sample-specific CIF: Adaptive Model Selection of BM} \label{sec:samplespecificcif}
Based on our theoretical analysis in Section \ref{sec:cif} and Section \ref{sec:geometryBM}, BM uses the most confident information (i.e., $[\eta^{2-}]$) for approximating distributions in an expected sense.
However, for the distribution with specific samples, \emph{can CIF further recognize less-confident parameters and reduce them properly}?
Next, inspired by the general CIF, we introduce an adaptive network design for VBM based on given samples, which could automatically balance the model complexity of BM and the sample size. The data constrains the state of knowledge about the unknown distribution. Let $q(x)$ denote the sampling distribution (representing the data). In order to force the estimate of our probabilistic model (denoted as $p(x;\xi)$) to meet the data, we could incorporate the data into CIF by recognizing the confidence of parameters $\xi$ in terms of $q(x)$. Then, parametric reduction procedure can be further applied to modify the topology of VBM adaptively according to the data, as shown in Algorithm \ref{alg:vbm} and explained as in the following. Note that this algorithm can also be used in BM with hidden units, such as vRBM (see Section \ref{sec:experimentrbm}).
\begin{algorithm}
\caption{Adaptive Network Design for BM}
\label{alg:vbm}
\begin{algorithmic}
\REQUIRE Samples $D=\{d_1, d_2, \dots, d_N\}$; Significance level $\alpha$; Nodes $V=\{x_1, x_2, \dots, x_n\}$; Edges $U=\{U_{ij}, \forall x_i, x_j\}$;
\ENSURE Set of confident edges $U_{conf} \subset U$
\STATE $U_{conf} \leftarrow \{\}; \alpha \leftarrow 0.05$
\FOR{$ U_{ij} \in U$}
\STATE Estimate marginal distribution $p(x_i, x_j)$ from samples
\STATE \textbf{**} parameterize to $\zeta$-coordinates: $[\zeta]$ \textbf{**}
\STATE $\eta_i \leftarrow E_p[x_i]; \eta_j \leftarrow E_p[x_j]$
\STATE $\theta^{ij} \leftarrow \log p_{00} - \log p_{01} - \log p_{10} + \log p_{11}$
\STATE $[\zeta] \leftarrow \{\eta_i, \eta_j, \theta^{ij}\}$
\STATE
\STATE \textbf{**} Fisher information of $\theta^{ij}$ in $[\zeta]$ \textbf{**}
\STATE $g \leftarrow (\frac{1}{p_{00}}+\frac{1}{p_{01}}+\frac{1}{p_{10}}+\frac{1}{p_{11}})^{-1}$
\STATE
\STATE \textbf{**} confidence of $\theta^{ij}$ in $[\zeta]$ \textbf{**}
\STATE $\rho_{ij} \leftarrow \theta^{ij} \cdot g\cdot \theta^{ij}$
\STATE
\STATE \textbf{**} hypothesis test: $\rho_{ij} = 0$ against $\rho_{ij} \neq 0$ \textbf{**}
\STATE $\pi \leftarrow cdf_{\chi^2(1)}(N\rho_{ij})$
\IF{$ (1-\pi)\cdot 2 < \alpha $}
\STATE \textbf{**} reject null hypothesis: $\rho_{ij} = 0$ \textbf{**}
\STATE $U_{conf} \leftarrow U_{conf} \cup \{U_{ij}\}$
\ENDIF
\ENDFOR
\RETURN $U_{conf}$
\end{algorithmic}
\end{algorithm}

As a graphical model, the VBM comprises a set of vertices $V = \{x_1, x_2, \dots, x_n\}$ together with a set of connections $U=\{U_{ij}, \forall x_i, x_j, i\neq j\}$.
The confidence for each connection parameter $U_{ij}$ can be assessed by the parameter selection criterion in CIF, i.e., the contribution to the Fisher information distance.
Based on the Theorem 1 in \cite{amari92igbm}, $U_{ij}$ could be expressed as follows:
$$U_{ij}=\log \frac{p(x_i = x_j = 1 | A) \cdot p(x_i = x_j = 0| A)}{p(x_i = 1, x_j = 0 | A) \cdot p(x_i = 0, x_j = 1| A)}$$
where the relation hold for any conditions $A$ on the rest variables. However, it is often infeasible for us to calculate the exact value of $U_{ij}$ because of data sparseness. To tackle this problem, we propose to approximate the value of $U_{ij}$ by using the marginal distribution $p(x_i, x_j)$ to avoid the effect of condition $A$.

Let $[\zeta]_{ij}=(\eta_{i}, \eta_{j}, \theta^{ij})$ be the mixed-coordinates for the marginal distribution $p(x_i, x_j)$ of VBM. Note that each $\theta^{ij}$ corresponds to one connection $U_{ij}$. Since $\theta^{ij}$ is orthogonal to $\eta_i$ and $\eta_j$, the Fisher information distance between two distributions can be decomposed into two independent parts: the information distance contributed by $\{\eta_i, \eta_j\}$ and $\{\theta^{ij}\}$. For the purpose of parameter reduction, we consider the two close distributions $p_1$ and $p_2$ with coordinates $\zeta_1 = \{\eta_i, \eta_j, \theta^{ij}\}$ and $\zeta_2 = \{\eta_i, \eta_j, 0\}$ respectively. The confidence of $\theta^{ij}$, denoted as $\rho(\theta^{ij})$, can be estimated by its contribution to Fisher information distance between $p_1$ and $p_2$:
\begin{equation}\label{eq:fisherdistance2}
  \rho(\theta^{ij}) = (\zeta_1-\zeta_2)^T G_\zeta (\zeta_1-\zeta_2) = \theta^{ij} \cdot g_\zeta(\theta^{ij}) \cdot \theta^{ij}
\end{equation}
where $G_\zeta$ is the Fisher information matrix in Proposition \ref{prop:fishermatrix_mix} and $g_\zeta(\theta^{ij})$ is the Fisher information for $\theta^{ij}$. Note that the second equality holds since $\theta^{ij}$ is orthogonal to $\eta_i$ and $\eta_j$.

To decide whether the Fisher information distance in the coordinate direction of $\theta^{ij}$ is significant or negligible, we set up the hypothesis test for $\rho$, i.e., null hypothesis $\rho = 0$ versus alternative $\rho \neq 0$. Based on the analysis in \cite{Nakahara02informationgeometric}, we have $N\rho \sim \chi^2(1)$ asymptotically, where the $\chi^2(1)$ is chi-square distribution with degree of freedom 1 and $N$ is the sampling number. For example, for coordinate $\theta^{ij}$ with $N\rho=5.024$, we have a $95\%$ confidence to reject the null hypothesis. That is, we can ensure that the Fisher information distance in the direction of $\theta^{ij}$ is significant with probability 95\%.
\subsection{Experiments with VBM}\label{sec:experimentsbm}
In this section, we investigate the density estimation performance of CIF-based model selection methods for VBM. Three methods are compared:
\begin{enumerate}
  \item Rand-CV: we perform random selection of VBM's connections and the best model is selected based on $k$-fold cross validation.
  \item CIF-CV: connections are selected in descend order based on their confidences (defined in Equation \ref{eq:fisherdistance2}), constrained on the number of model free parameters. Then, the best model is selected based on $k$-fold cross validation.
  \item CIF-Htest: the topology of VBM is determined by the adaptive algorithm described in Algorithm \ref{alg:vbm}.
\end{enumerate}
%
%

\subsubsection{Experimental on artificial dataset}
~\\
\vspace{-4mm}
\par

The artificial binary dataset is generated as follows: we first randomly select the target distribution $q(x)$, which is randomly chosen from the open probability simplex over the $n$ random variables using the Jeffreys prior \cite{jeffreys1946invariant}. Then, the dataset with $N$ samples are generated from $q(x)$.
For computation simplicity, the artificial dataset is set to be 10-dimensional.
The CD learning algorithm is used to train the VBMs.

The Full-VBM, i.e., the VBM with full connections are used as baseline. $k$ is set to 5 for cross validation.
KL divergence is used to evaluate the goodness-of-fit of the VBM trained by various algorithms. For sample size $N$, we run 20 randomly generated distributions and report the averaged KL divergences.
Note that we focus on the case that the variable number is relatively small ($n\!\!=\!\!10$) in order to analytically evaluate the KL divergence and give a detailed study on algorithms. Changing the number of variables only offers a trivial influence for experimental results since we obtained qualitatively similar observations on various variable numbers (not reported here).

\textbf{Results and Summary:}
\begin{figure}
  \centering
  \includegraphics[width=0.5\textwidth]{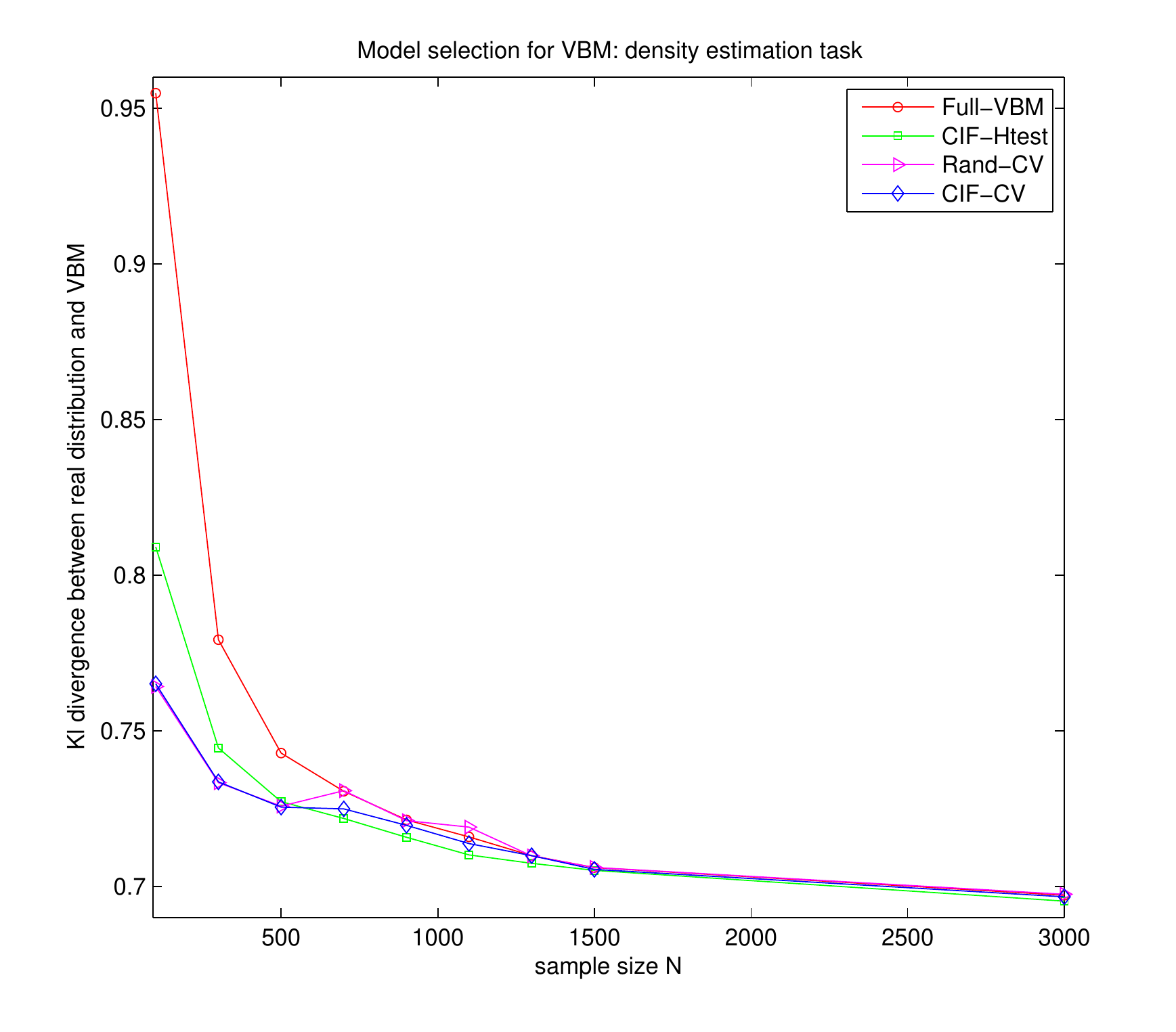}
  \caption{Density estimation results for VBM}
  \label{fig:result_vbm} 
\end{figure}
\begin{figure*}
  \centering
  \includegraphics[width=1.1\textwidth]{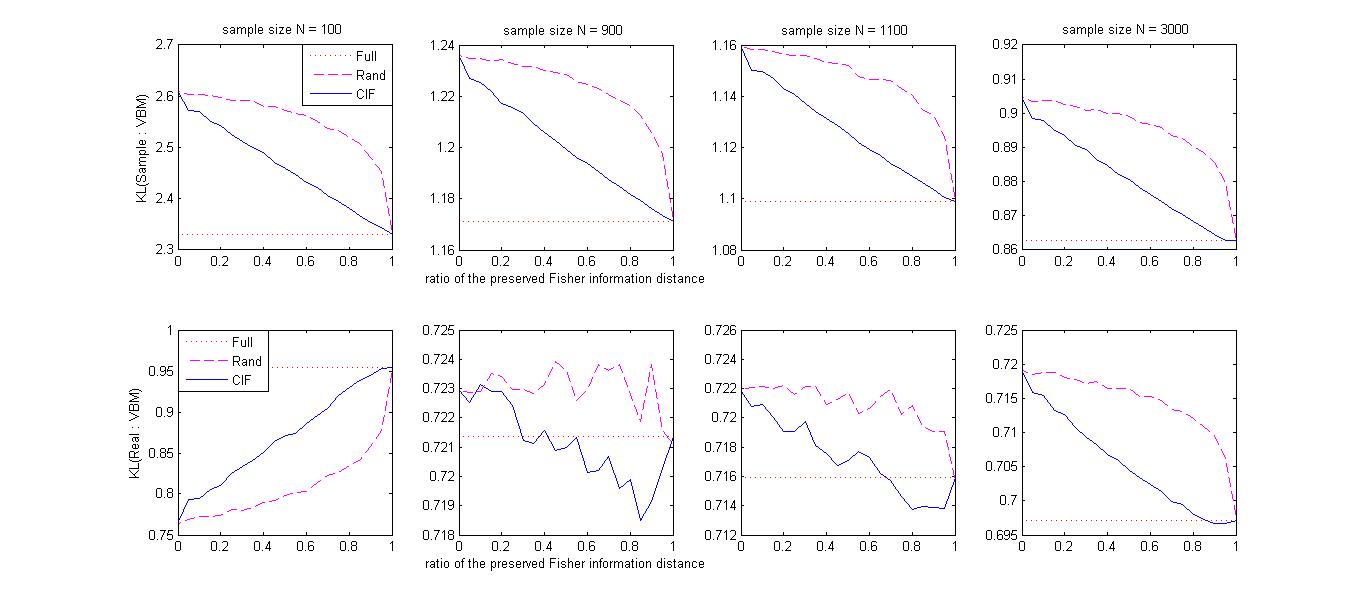}
  \caption{Changes of KL divergence between VBM and sample/real distributions w.r.t. model complexity. The model complexity is measured by the ratio $=($ sum of confidences for preserved connections$)/($ total confidences for all connections$)$}
  \label{fig:result_vbm_change} 
\end{figure*}
The averaged KL divergences between VBM and the underlying real distribution are shown in Figure \ref{fig:result_vbm}.
We can see that all model selection methods could improve density estimation results of VBM, especially when the sample size is small (N=100 to 1100). With relatively large samples, the effect of parameter reduction gradually becomes marginal.

Comparing CIF-CV with Rand-CV, the performances on relatively small sample size (N=100, 300, 500) are similar (see Figure \ref{fig:result_vbm}). This is because that the VBM reduced by cross validation is the trivial one with no connections. In Figure \ref{fig:result_vbm_change} (first column), we illustrate how the KL divergence between VBM and real/sample distribution changes along with different model complexities of VBM. We can see that although the CIF is worse than Rand in terms of the KL divergence between VBM and the real distribution in most setups of model complexity, CIF is better in estimating the sampling distribution. This is consistent with our previous theoretical insight that CIF preserves the most confident parameters in VBM where the confidence is estimated based on sampling distribution.

As sample size increases (N=700, 900, 1100), the CIF-CV gradually outperforms Rand-CV (see Figure \ref{fig:result_vbm}). This could be explained by the CIF principle. The CIF preserves the most confident parameters of VBM with respect to sample distribution. As sample size increases, the sampling distribution grows closer to real distribution, which could benefit CIF in the way that the KL divergence with both sample/real distributions can be simultaneous better than Rand, as shown in Figure \ref{fig:result_vbm_change} (second and third column).

With relatively large sample size (N=1500, 3000), the CIF-CV and Rand-CV have similar performance in terms of the KL divergence between VBM and the real distribution (see Figure \ref{fig:result_vbm}). This is because complex model is preferred when the samples are sufficient and both CIF-CV and Rand-CV tend to select the trivial VBM with full connections. Therefore, the difference between CIF-CV and Rand-CV becomes marginal. However, CIF is still more powerful in describing sample/real distributions compare to Rand, as shown in Figure \ref{fig:result_vbm_change} (fourth column).

For the two CIF-based algorithm, CIF-Htest is worse than CIF-CV when sample size is small and gradually achieves similar or better performance with CIF-CV along with the increasing sample size. The main advantage of CIF-Htest, w.r.t., CIF-CV, is that there is no need for the time-consuming cross validation.

In summary, the CIF-based model selection (e.g., CIF-CV and CIF-Htest) indicates the balance between VBM's model complexity and the amount of information learnt from samples. For nontrivial cases of model selection (trivial cases are the selected VBMs with no connections or full connections), CIF-based methods could reduce the KL divergence between VBM and real distributions without hurting the VBM too much in describing sample distribution, as compared to Rand (Fig. \ref{fig:result_vbm_change}). Moreover, the CIF-Htest provides us an automatic way to adaptively select suitable VBM with respect to given samples.

%

\subsubsection{Experiments on real datasets} \label{sec:realdata_vbm}
~\\
\vspace{-4mm}
\par
\begin{figure}
  \centering
  \includegraphics[width=0.4\textwidth]{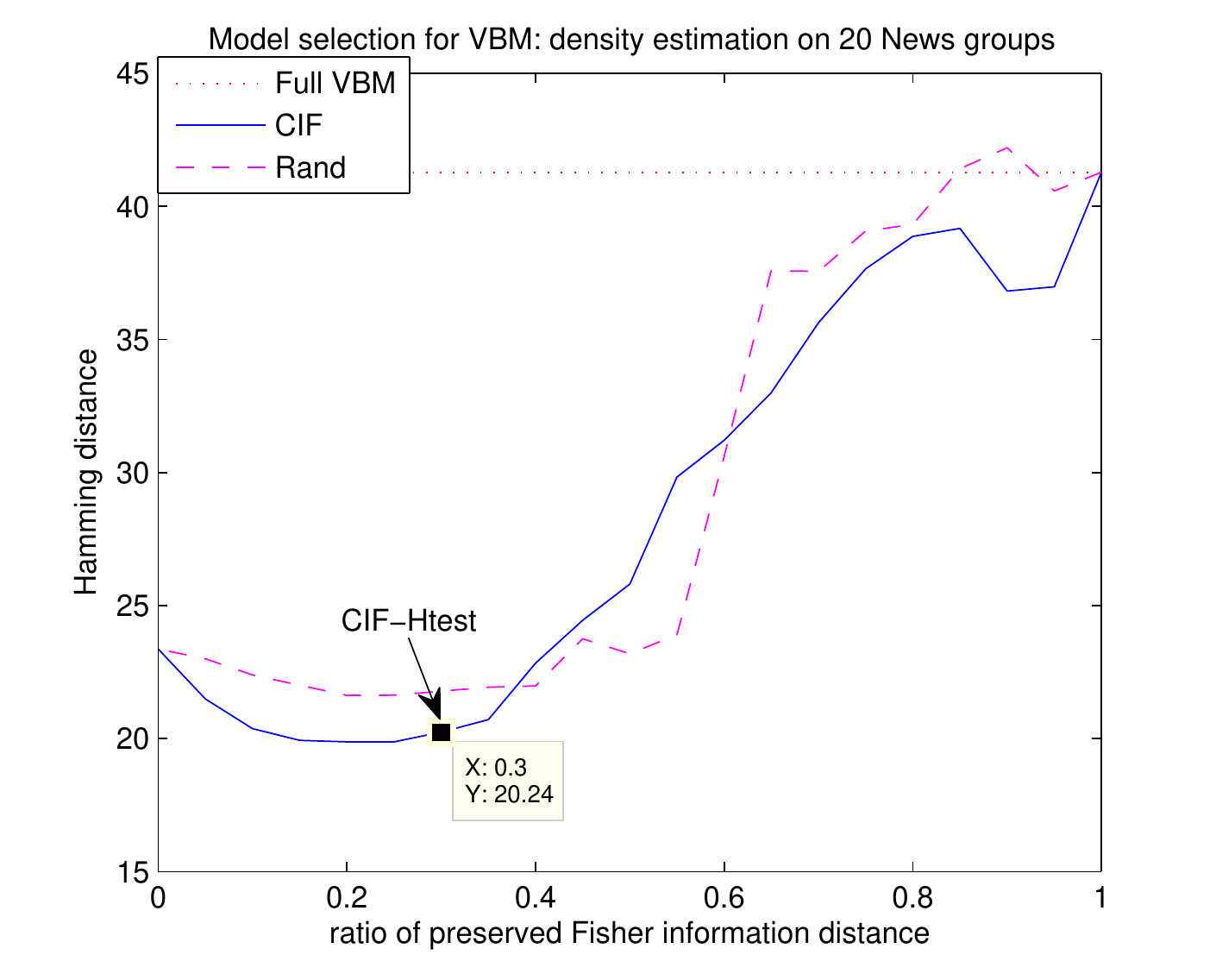}
  \caption{Performance changes on real dataset w.r.t. model complexity}
  \label{fig:20newsgroup} 
\end{figure}
In this section, we empirically investigate how the CIF-based model selection algorithm works on real-world datasets in the context of density estimation. In particular, we use the VBM to learn the underlying probability density over 100 terms of the \emph{20 News Groups} binary dataset, with different model complexities (changing the ratio of preserved Fisher information distance). There are 18000 documents in 20 News Groups in total, which is partitioned into two set: train set ($80\%$) and test set ($20\%$).The learning rate for CD is manually tuned in order to converge properly and all set to 0.01. Since it is infeasible to compute the KL devergence due to the high dimensionality, the averaged Hamming distance between the samples in the dataset and those generated from the VBM is used to evaluate the goodness-of-fit of the VBM's trained by various algorithms. Let $D=\{d_1, d_2,\dots, d_N\}$ denote the dataset of $N$ documents, where each document $d_i$ is a 100-dimensional binary vector. To evaluate a VBM with parameter $\xi_{vbm}$, we first randomly generate $N$ samples from the stationary distribution $p(x;\xi_{vbm})$, denoted as $V=\{v_1, v_2,\dots, v_N\}$. Then the averaged hamming distance $D_{ham}$ is calculated as follows:
$$D_{ham}[D,V]=\frac{\sum_{d_i} (\min_{v_j} (Ham[d_i, v_j])}{N}$$
where $Ham[d_i, v_j]$ is the number of positions at which the corresponding values are different.

Three kinds of VBMs are compared: the full VBM without parametric reduction; the VBMs of different model complexities using Rand; the VBMs of different model complexities using CIF. After training all VBMs on the training dataset, we evaluate the trained VBM on the test dataset. The result is shown in Figure \ref{fig:20newsgroup}. We also mark the VBM that is automatically selected by CIF-Htest.
We can see that the model selection (CIF and Rand) achieves significantly better performances than the full VBM for a wide range of $r$, which is consistent with our observations with the experiments on artificial datasets when the samples is insufficient. And the best performance for CIF outperforms that of Rand. The performance of CIF-Htest (ratio=0.3) is also shown in Fig \ref{fig:20newsgroup}, which is close to the optimal solution (ratio=0.2).

\subsection{Experiments with vRBM} \label{sec:experimentrbm}
\begin{figure*}
  \centering
  \includegraphics[width=1\textwidth]{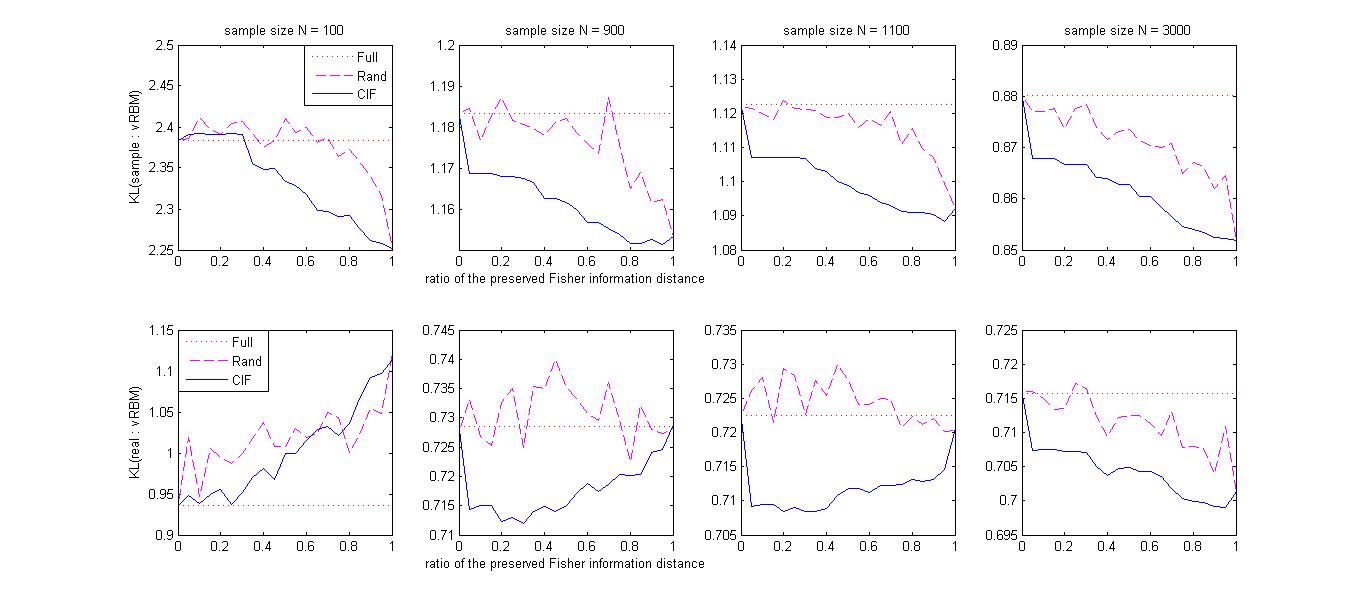}
  \caption{Changes of KL divergence between vRBM and sample/real distributions w.r.t. model complexity. The model complexity is measured by the ratio $=($ sum of confidences for preserved visible connections$)/($ total confidences for all visible connections$)$}
  \label{fig:result_vrbm_change} 
\end{figure*}
The BM with hidden units is practically more interesting than VBM, since it has higher representation power.
Particularly, one of the fundamental problem in neural network research is the unsupervised representation learning \cite{Bengio12representationlearning}, which attempts to characterize the underlying distribution through the discovery of a set of latent variables (or features). The restricted BM (RBM) is one of the most widely used models for learning one level of feature extraction. Then, in deep learning models\cite{hinton06}, the representation learnt at one level is used as input for learning the next level.
In this section, we will extend the RBM by allowing connections among visible units (called vRBM, shown in Figure \ref{fig:vRBM}) and further investigate the CIF-based model selection algorithm (see Section \ref{sec:samplespecificcif}) empirically. Note that in the model selection for vRBM, only parameters in $U$ (connections within visible units) are affected and all parameters in $\{W,b,d\}$ (connections between hidden and visible units, visible/hidden self-connections) are preserved, so as to maintain the structure of vRBM.


\subsubsection{Experimental Setup}
~\\
\vspace{-4mm}
\par
For computational simplicity, the artificial dataset is of 10 dimensionality, and the number of hidden units in vRBM is set to 10.
For the model selection of vRBM, three methods are compared: CIF-CV, Rand-CV, CIF-Htest, which are the same model selection scheme for VBM in Section \ref{sec:experimentsbm}. The standard RBM is adopted as baseline. Note that for model selection, we are actually adding connections among visible units to the standard RBM.
The learning algorithms is CD. KL divergence is used to evaluate the goodness-of-fit of the vRBM's trained by various algorithms.
\subsubsection{Results and Summary}
~\\
\vspace{-4mm}
\par
\begin{figure}
  \centering
  \includegraphics[width=0.5\textwidth]{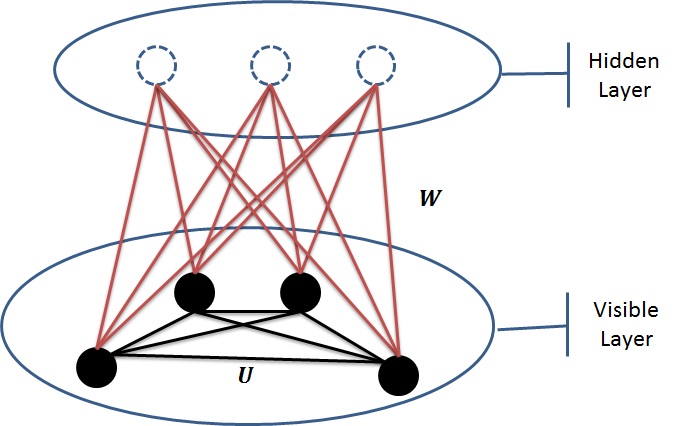}
  \caption{The network structure of vRBM}
  \label{fig:vRBM} 
\end{figure}
\begin{figure}
  \centering
  \includegraphics[width=0.5\textwidth]{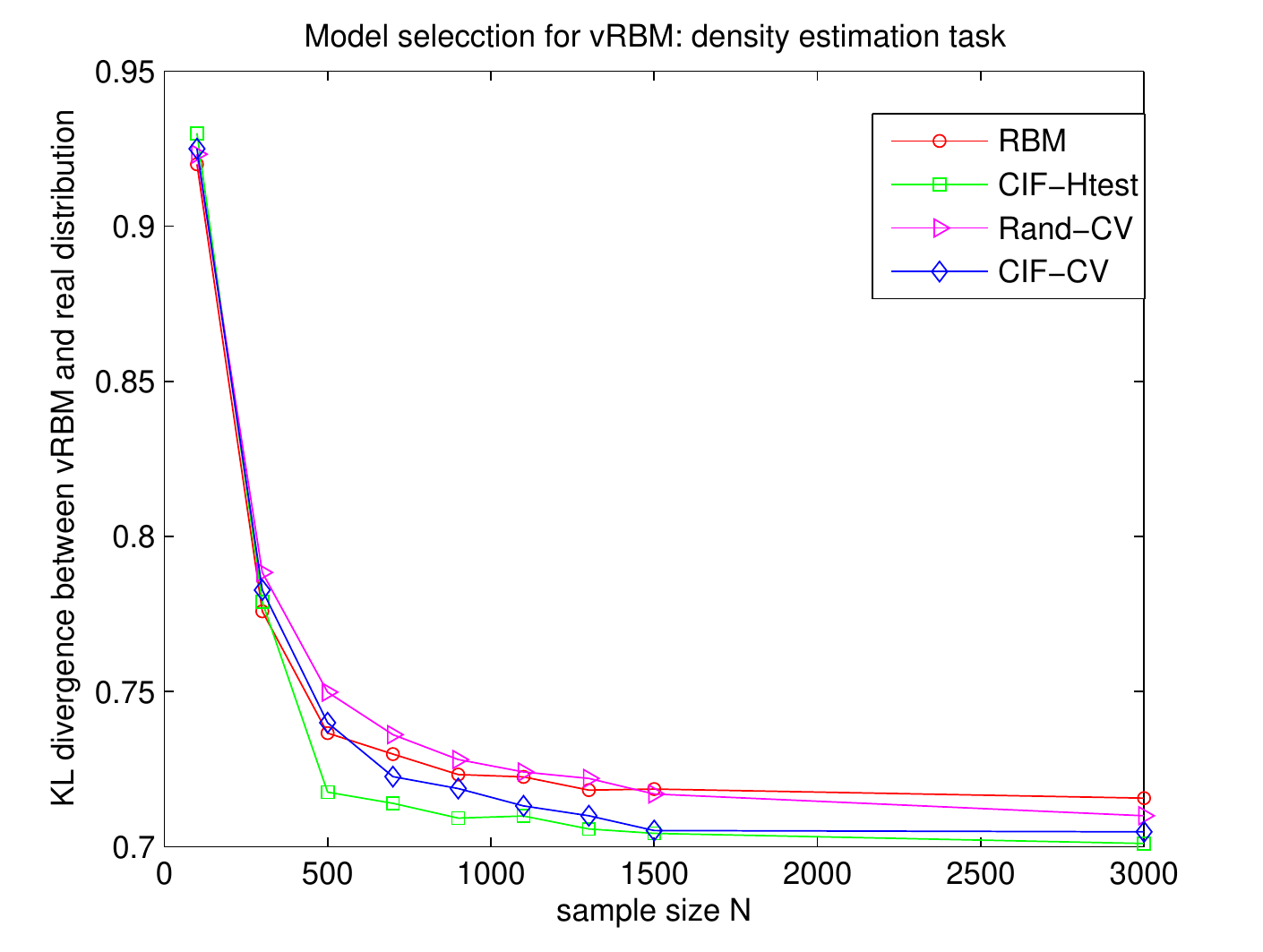}
  \caption{Density estimation results of vRBM}
  \label{fig:result_vrbm} 
\end{figure}
The averaged KL divergences between vRBM and the underlying distribution are shown in Figure \ref{fig:result_vrbm}.
For a small sample size (N=100,300), there is not so much need to increase the model complexity of BM and hence adding connections among visible units does not improve the density estimation results. While, as sample size increases (from 500 to 3000), model selection methods gradually outperforms standard RBM significantly by adding connections among visible units.

Comparing CIF-CV with Rand-CV, the performances on relatively small sample size (N=100, 300) are similar (see Figure \ref{fig:result_vrbm}). This is because that the vRBM selected by cross validation is the trivial one with no connections added to RBM. In Figure \ref{fig:result_vrbm_change} (first column), we illustrate how the KL divergence between vRBM and real/sample distribution changes along with different model complexities of vRBM. We can see that CIF is better in estimating the sample distribution compared with Rand (first column, first row). Similary with VBM, the performance for estimating the real distribution of CIF becomes similar or worse than Rand along with the increasing model complexity (ratio$>0.6$) due to overfitting (first column, second row).

As sample size increases (N=500 to 1500), the CIF-CV gradually outperforms Rand-CV (see Figure \ref{fig:result_vrbm}). This could be explained by the CIF principle. The CIF-CV preserves the most confident parameters of vRBM with respect to sample distribution. As sample size increases, the sampling distribution grows closer to real distribution, which could benefit CIF in the way that the KL divergence with both sample/real distributions can be simultaneous better than Rand for all model complexities, as shown in Figure \ref{fig:result_vrbm_change} (second and third column).

With larger sample size (N$>=$3000), the CIF-CV and Rand-CV have similar performance in terms of the KL divergence between vRBM and the real distribution. This is because complex model is preferred when the samples are sufficient and both CIF-CV and Rand-CV tend to select the trivial vRBM with all connections between visible units being added to RBM. Therefore, the difference between CIF-CV and Rand-CV becomes marginal. However, CIF are still more powerful in describing sample distribution compare to Rand for the vRBM with the same number of connections, as shown in Figure \ref{fig:result_vrbm_change} (fourth column).

For the two CIF-based algorithm, CIF-Htest is worse than CIF-CV when sample size is small and gradually outperforms CIF-CV along with the increasing sample size.

In summary, the CIF-based model selection could balance between vRBM's model complexity and the amount of information learnt from samples and could simultaneously reduce the KL divergence between vRBM and real/sample distributions. This indicates that the CIF is also useful for BM with hidden units.
%

\section{Conclusions and Future works} \label{sec:conclusions}
In this paper, we study the parametric reduction and model selection problem of Boltzmann machines from both theoretical and applicational perspectives. On the theoretical side, we propose the CIF principle for the parametric reduction to maximally preserve the confident parameters and ruling out less confident ones. For binary multivariate distributions, we theoretically show that CIF could lead to an optimal submanifold in terms of Equation \ref{eq:parametricreduction2}. Furthermore, we illustrate that the Boltzmann machines (with or without hidden units) can be derived from the general manifold based on CIF principle. In future works, the CIF could be the start of an information-oriented interpretation of deep learning models where BM is used as building blocks. For deep Boltzmann machine (DBM) \cite{Salakhutdinov2012}, several layers of RBM compose a deep architecture in order to achieve a representation at a sufficient abstraction level. The CIF principle describes how the information flows in those representation transformations, as illustrated in Figure \ref{fig:deeparchitecture}. We propose that each layer of DBM determines a submanifold $M$ of $S$, where $M$ could maximally preserve the highly confident information on parameters. Then the whole DBM can be seen as the process of repeatedly applying CIF in each layer, achieving the tradeoff between the abstractness of representation features and the intrinsic information confidence preserved on parameters. The more detailed analysis on deep models will be left as further works.
\begin{figure}
  \centering
  \includegraphics[width=0.5\textwidth]{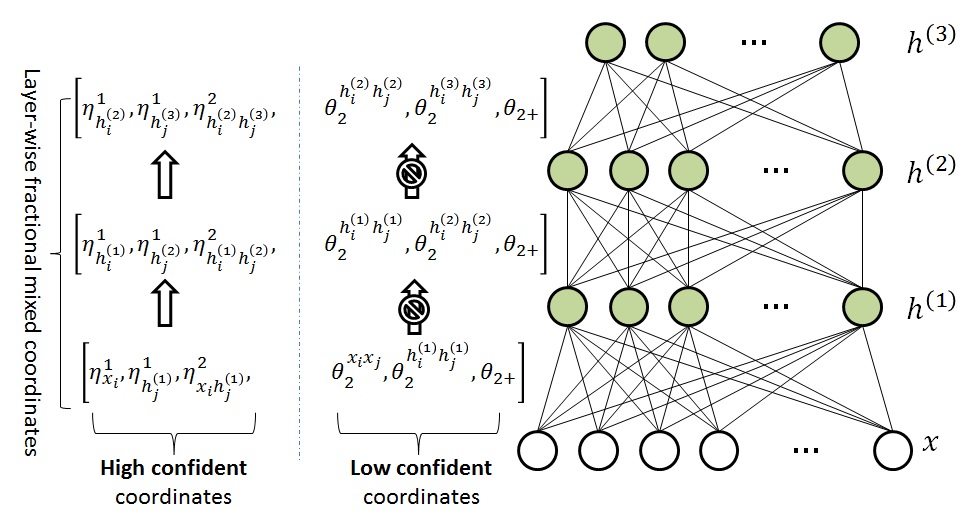}
  \caption{A multi-layer BM with visible units $x$ and hidden layers $h^{(1)}$, $h^{(2)}$ and $h^{(3)}$. The greedy layer-wise training of deep architecture is to maximally preserve the confident information layer by layer. Note that the prohibition sign indicates that the Fisher information on lowly confident coordinates is not preserved.}
  \label{fig:deeparchitecture} 
\end{figure}

On the applicational side, we propose a sample-specific CIF-based model selection scheme for BM, i.e., CIF-Htest, that could automatically adapt to the given samples. It is studied in a series of density estimation experiments. In the further work, we plan to incorporate the CIF-Htest into deep learning models (such as DBM) to modify the network topology such that the most confident information in data can be well captured.

\appendices
\section{Proof of Proposition \ref{prop:fishermatrix}} \label{appendix:thetafisher}
\begin{proof}
By definition, we have:
\begin{equation}\label{eq:propFisher1}
 g_{IJ}=\frac{\partial^2 \psi(\theta)}{\partial \theta ^I \partial \theta ^J} \nonumber
\end{equation}
where $\psi(\theta)$ is defined by Equation (\ref{eq:legedre}). Hence, we have:
\begin{eqnarray}\label{eq:propFisher2}
 g_{IJ}=\frac{ \partial^2 (\sum_{I} {\theta^I \eta_I} - \phi(\eta))}{\partial \theta^I \partial \theta^J} = \frac{\partial \eta_I}{\partial \theta^J} \nonumber
\end{eqnarray}
By differentiating $\eta_I$, defined by Equation (\ref{eq:etacoordinate}), with respect to $\theta^J$, we have:
\begin{eqnarray}\label{eq:propFisher3}
 g_{IJ}\!\!&=&\!\!\frac{\partial \eta_I}{\partial \theta^J} = \frac{\partial \sum_x X_I(x)(exp\{\sum_I{\theta^I X_I(x)} - \psi(\theta)\})}{\partial \theta^J}\nonumber \\
 \!\!&=&\!\! \sum_x {X_I(x) [X_J(x) - \eta_J] p(x;\theta)} = \eta_{I\cup J}-\eta_I \eta_J \nonumber
\end{eqnarray}
This completes the proof.
\end{proof}

\section{Proof of Proposition \ref{prop:fishermatrix_eta}} \label{appendix:etafisher}
\begin{proof}
By definition, we have:
\begin{equation}\label{eq:propEtaFisher}
 g^{IJ} = \frac{\partial^2 \phi(\eta)}{\partial \eta_I \partial \eta_J} \nonumber
\end{equation}
where $\phi(\eta)$ is defined by Equation (\ref{eq:legedre}). Hence, we have:
\begin{eqnarray}\label{eq:propEtaFisher2}
 g^{IJ} &=& \frac{\partial^2 (\sum_{J} {\theta^J \eta_J} - \psi(\theta))}{\partial \eta_I \partial \eta_J} = \frac{\partial \theta^I}{\partial \eta_J} \nonumber
\end{eqnarray}
Based on Equations (\ref{eq:thetacoordinate}) and (\ref{eq:etacoordinate}), the $\theta^I$ and $p_{K}$ could be calculated by solving a linear equation of $[p]$ and $[\eta]$, respectively. Hence, we have:
\begin{equation}\label{eq:propThetaP}
 \theta^I = \sum_{K \subseteq I} (-1)^{|I-K|} log(p_{K});~~p_{K} = \sum_{K \subseteq J} (-1)^{|J-K|} \eta_{J} \nonumber
\end{equation}
Therefore, the partial derivation of $\theta^I$ with respect to $\eta_J$ is:
\begin{eqnarray}\label{eq:propEtaFisher3}
 g^{IJ}\!\!=\!\!\frac{\partial \theta^I}{\partial \eta_J}\!\! =\!\!\sum_{K} \frac{\partial \theta^I}{\partial p_{K}} \cdot \frac{\partial p_{K}}{\partial \eta_J}\!\! =\!\! \sum_{K\subseteq I\cap J}{(-1)^{|I-K|+|J-K|} \cdot \frac{1}{p_{K}}} \nonumber
\end{eqnarray}
This completes the proof.
\end{proof}

\section{Proof of Proposition \ref{prop:fishermatrix_mix}} \label{appendix:mixfisher}
\begin{proof}
The Fisher information matrix of $[\zeta]$ could be partitioned into four parts: $G_\zeta= \left(
    \begin{array}{cc}
     A & C \\
     D & B \\
    \end{array}
    \right)$.
It can be verified that in the mixed coordinate, the $\theta$-coordinate of order $k$ is orthogonal to any $\eta$-coordinate less than $k$-order, implying the corresponding element of the Fisher information matrix is zero ($C=D=0$) \cite{Nakahara02informationgeometric}. Hence, $G_\zeta$ is a block diagonal matrix.

According to the Cram\'{e}r--Rao bound \cite{rao45attainable}, a parameter (or a pair of parameters) has a unique asymptotically tight lower bound of the variance (or covariance) of the unbiased estimate, which is given by the corresponding element of the inverse of the Fisher information matrix involving this parameter (or this pair of parameters). Recall that $I_\eta$ is the index set of the parameters shared by $[\eta]$ and $[\zeta]_l$ and that $J_\theta$ is the index set of the parameters shared by $[\theta]$ and $[\zeta]_l$; we have $(G_\zeta^{-1})_{I_\zeta} = (G_\eta^{-1})_{I_\eta}$ and $(G_\zeta^{-1})_{J_\zeta} = (G_\theta^{-1})_{J_\theta}$, \textit{i.e}.,
$$G_{\zeta}^{-1}= \left(
    \begin{array}{cc}
     (G_\eta^{-1})_{I_\eta} & 0 \\
     0 & (G_\theta^{-1})_{J_\theta} \\
    \end{array}
    \right)$$
Since $G_\zeta$ is a block tridiagonal matrix, the proposition follows.
\end{proof}

\section{Proof of Proposition \ref{prop:GeometricView}} \label{appendix:geometriccif}

\begin{proof}
Let $B_q$ be a $\varepsilon$-ball surface centered at $q(x)$ on manifold $S$, \textit{i.e}., $B_q= \{q' \in S |\| KL(q,q') = \varepsilon\}$, where $KL(\cdot,\cdot)$ denotes the Kullback--Leibler divergence and $\varepsilon$ is small. $\zeta_q$ is the coordinates of $q(x)$. Let $q(x)+dq$ be a neighbor of $q(x)$ uniformly sampled on $B_q$ and $\zeta_{q(x)+dq}$ be its corresponding coordinates. For a small $\varepsilon$, we can calculate the expected square Fisher information distance between $q(x)$ and $q(x)+dq$ as follows:
\begin{equation}\label{eq:expectdistance}
 E_{B_q}=\int (\zeta_{q(x)+dq}-\zeta_q)^T G_\zeta (\zeta_{q(x)+dq}-\zeta_q) dB_q
\end{equation}
where $G_\zeta$ is the Fisher information matrix at $q(x)$.

Since Fisher information matrix $G_\zeta$ is both positive definite and symmetric, there exists a singular value decomposition $G_\zeta = U^T \Lambda U$ where $U$ is an orthogonal matrix and $\Lambda$ is a diagonal matrix with diagonal entries equal to the eigenvalues of $G_\zeta$ (all $\geq 0$), i.e., $(\lambda_1, \lambda_2, \dots, \lambda_n)$.

Applying the singular value decomposition into Equation (\ref{eq:expectdistance}), the expectation becomes:
\begin{equation}\label{eq:expectdistance2}
 E_{B_q}\!\!=\!\!\!\!\int (\zeta_{q(x)+dq}-\zeta_q)^T U^T \Lambda U (\zeta_{q(x)+dq}-\zeta_q) dB_q
\end{equation}
Note that $U$ is an orthogonal matrix, and the transformation $U (\zeta_{q(x)+dq}-\zeta_q)$ is a norm-preserving rotation.

Now, we need to show that among all tailored $k$-dimensional submanifolds of $S$, $[\zeta]_{l_t}$ is the one that preserves maximum information distance. Assume $I_{T}=\{i_1, i_2, \dots, i_k\}$ is the index of $k$ coordinates that we choose to form the tailored submanifold $T$ in the mixed-coordinates $[\zeta]$.

First, according to the fundamental analytical properties of the surface of the hyper-ellipsoid, we will show that there exists a strict positive monotonicity between the expected information distance $E_{B_q}$ for $T$ and the sum of eigenvalues of the sub-matrix $(G_\zeta)_{I_T}$.

Based on the definition of hyper-ellipsoid, the $\varepsilon$-ball surface $B_q$ is indeed the surface of a hyper-ellipsoid (centered at $q(x)$) determined by $G_\zeta$. Let the eigenvalues of $G_\zeta$ be $\lambda_1 > \lambda_2 > \dots > \lambda_n > 0$. Then, the surface integral on $B_q$ can be decomposed as follows:
\begin{eqnarray}
  E_{B_q} &=& \int v^T \Lambda v dB_q \nonumber \\
   &=& \int v^T diag(\lambda_1) v dB_q + \int v^T diag(\lambda_2) v dB_q \nonumber\\
   && + \dots + \int v^T diag(\lambda_n) v dB_q \nonumber
\end{eqnarray}
where $v=(\zeta_{q(x)+dq}-\zeta_q)U$ and $diag(\lambda_i)$ is the diagonal matrix with the main diagonal $(0,\dots,\lambda_i,\dots,0)$.
We can show the monotonicity between the integration values and eigenvalues:
\begin{equation}\label{eq:monotonicinequality}
  \int_{B_q} v^T diag(\lambda_1) v > \int_{B_q} v^T diag(\lambda_2) v > \dots > \int_{B_q} v^T diag(\lambda_n) v
\end{equation}

The surface of the ellipsoid $B_q$ may be parameterized in several ways. In terms of Cartesian coordinate systems, the equation of $B_q$ is:
$$\frac{v_1^2}{r_1^2}+\frac{v_2^2}{r_2^2}+\dots+\frac{v_n^2}{r_n^2}=\varepsilon$$
where $r_i^2$ denotes the squares of the semi-axes and is determined by the reciprocals of the eigenvalues, i.e., $ r_i^2= \frac{1}{\lambda_i}$.

Consider the two-dimensional ellipsoid $B_q$:
$$\frac{v_1^2}{r_1^2}+\frac{v_2^2}{r_2^2}=\varepsilon, (r_1 < r_2)$$
Then we can transform the Cartesian coordinates to Spherical coordinates as follows:
$$\left\{
    \begin{array}{ll}
      v_1= r_1 \cdot \cos \theta \\
      v_2 = r_2 \cdot \sin \theta
    \end{array}
  \right.
$$
where $0 \le \theta \le 2\pi$. To prove that $\int_{B_q} v^T diag(\lambda_1) v > \int_{B_q} v^T diag(\lambda_2) v$, we need to show that:
$\int_{B_q} |\cos \theta| > \int_{B_q} |\sin \theta|$. Since the ellipsoid is symmetric, we only need to prove that:
$$\int_{B_q} \cos \theta > \int_{B_q} \sin \theta, \theta \in [0,\frac{\pi}{2}]$$

By reformulating the above surface integral in terms of definite integral, we have:
\begin{equation}\label{eq:definiteintegral}
  \int_0^{\frac{\pi}{2}} \!\! \cos \theta \sqrt{r_1^2 \!\!+\!\! (r_2^2 \!\!- \!\!r_1^2)\cos^2 \theta}  > \!\! \int_0^{\frac{\pi}{2}} \!\! \sin \theta \sqrt{r_1^2 \!\!+\!\! (r_2^2\!\! -\!\! r_1^2)\sin^2 \theta}
\end{equation}

Next we will prove the above inequality based on the definition of Riemann integral. The integral interval $[0,\frac{\pi}{2}]$ of $\theta$ can be partitioned into a finite sequence of subintervals, i.e., $[\theta_i,\theta_{i+1}]$, where $$0=\theta_0<\dots<\theta_i<\theta_{i+1}<\dots<\theta_m=\frac{\pi}{2}$$

Let $\beta = max_{0\le i < n} |\theta_{i+1} - \theta_i|$ be the maximum length of subintervals. When $\beta$ approaches infinitesimal (hence $n \rightarrow \infty$), the definite integral equals to the Riemann integral, i.e., the limit of the Riemann sums. Therefore, the inequality \ref{eq:definiteintegral} can be transformed in to the comparison between two limitations:
\begin{equation}\label{eq:definiteintegrallimit}
 \!\!\!\! \lim_{\beta \rightarrow 0} \overrightarrow{\cos \theta} \cdot \overrightarrow{\sqrt{r_1^2 \!\!+\!\! (r_2^2 \!\!- \!\!r_1^2)\cos^2 \theta}} > \lim_{\beta \rightarrow 0} \overrightarrow{\sin \theta} \cdot \overrightarrow{\sqrt{r_1^2 \!\!+\!\! (r_2^2 \!\!- \!\!r_1^2)\cos^2 \theta}}
\end{equation}
where the Riemann sums are denoted by vector multiplications and $\overrightarrow{\cos \theta}=(\cos \theta_0,\dots,\cos \theta_n)$ and $\overrightarrow{\sin \theta}=(\sin \theta_0,\dots,\sin \theta_n)$.
Since $\cos 0 = \sin \frac{\pi}{2}$ and $\cos \frac{\pi}{2} = \sin 0$, $\overrightarrow{\cos \theta}$ and $\overrightarrow{\sin \theta}$ can be seen as two vectors that share the same components while arranged in different orders. We can also see that there is a positive correlation between the components in $\overrightarrow{\cos \theta}$ and $\overrightarrow{\sqrt{r_1^2 \!\!+\!\! (r_2^2 \!\!- \!\!r_1^2)\cos^2 \theta}}$, while there is a negative correlation between the components in $\overrightarrow{\sin \theta}$ and $\overrightarrow{\sqrt{r_1^2 \!\!+\!\! (r_2^2 \!\!- \!\!r_1^2)\cos^2 \theta}}$. Then, the inequality \ref{eq:definiteintegral} is proved. Therefore, the monotonicity \ref{eq:monotonicinequality} holds for the two-dimensional ellipsoid.

%

Similarly, the above Reimann integral analysis on the spherical coordinates can be extended to the $n$-dimensional ellipsoid. Hence, the monotonicity \ref{eq:monotonicinequality} holds for standard ellipsoids.

Thus, based on the monotonicity \ref{eq:monotonicinequality}, the expected Fisher information distance that can be preserved by a $k$-dimensional standard ellipsoid is monotonic with the sum of eigenvalues of the selected $k$ eigenvectors. To maximize the preserved information distance, we should choose the top-$k$ eigenvalues, i.e., ${\lambda_1, \dots, \lambda_k}$.

Since $G_\zeta$ is a block diagonal matrix, the eigenvalues of $G_\zeta$ are the combined eigenvalues of its blocks. Based on Lemma \ref{prop:fishermatrix_mixDiagonal}, the elements on the main diagonal of the sub-matrix $A$ are lower bounded by one and those of $B$ upper bounded by one. Since the the sum of eigenvalues equals to the trace, the top-$k$ eigenvalues is exactly the eigenvalues of sub-matrix $A$. Thus, we have:
\begin{eqnarray}\label{eq:rotationinvairant}
  &&\!\!\!\!\!\!\!\!E_{max} =\int v^T diag(\lambda_1) v dB_q + \dots + \int v^T diag(\lambda_k) v dB_q \nonumber \\
   &=& \!\!\!\!\int (\zeta_{q(x)+dq}-\zeta_q)^T \left(
  \begin{array}{cc}
   A & 0 \\
   0 & 0 \\
  \end{array}
  \right) (\zeta_{q(x)+dq}-\zeta_q) dB_q
\end{eqnarray}

Now, let us consider the best selection of coordinates ${I_T}$ in $[\zeta]$. It is easy to see that the rotation operation $U$ in Equation \ref{eq:expectdistance2} does not affect the integral value of the expected Fisher information distance. Therefore, based on Equation \ref{eq:rotationinvairant}, $I_T = \{\eta^{2-}\}$ gives the maximum Fisher information distance. This completes the proof.
\end{proof}
\begin{lemma}\label{prop:fishermatrix_mixDiagonal}
For Fisher information matrix $G_\zeta$, the diagonal elements of $A$ are lower bounded by one, and those of $B$ are upper bounded by one.
\end{lemma}
\begin{proof}
Assume the Fisher information matrix of $[\theta]$ to be:
$G_\theta=\left(
  \begin{array}{cc}
   U & X \\
   X^T & V \\
  \end{array}
  \right)
$, which is partitioned based on $I_\eta$ and $J_\theta$. Based on Proposition \ref{prop:fishermatrix_mix}, we have $A=U^{-1}$. Obviously, the diagonal elements of $U$ are all smaller than one. According to the succeeding Lemma \ref{lemma:diagOfmix}, we can see that the diagonal elements of $A$ (\textit{i.e}., $U^{-1}$) are greater than one.

Next, we need to show that the diagonal elements of $B$ are smaller than $1$. Using the Schur complement of $G_\theta$, the bottom-right block of $G_\theta^{-1}$, \textit{i.e}., $(G_\theta^{-1})_{J_\theta}$, equals to $(V-X^TU^{-1}X)^{-1}$. Thus, the diagonal elements of B: $B_{jj}=(V-X^TU^{-1}X)_{jj}<V_{jj}<1$. Hence, we complete the proof.
\end{proof}
\begin{lemma} \label{lemma:diagOfmix}
With a $l\times l$ positive definite matrix $H$, if $H_{ii}<1$, then $(H^{-1})_{ii}>1, \forall i\in \{1,2,\dots,l\}$.
\end{lemma}
\begin{proof}
Since $H$ is positive definite, it is a Gramian matrix of $l$ linearly independent vectors $v_1,v_2,\dots,v_l$, \textit{i.e}., $H_{ij}=\langle v_i,v_j\rangle$ ($\langle \cdot,\cdot \rangle$ denotes the inner product). Similarly, $H^{-1}$ is the Gramian matrix of $l$ linearly independent vectors $w_1,w_2,\dots,w_l$ and $(H^{-1})_{ij}=\langle w_i,w_j\rangle$. It is easy to verify that $\langle w_i,v_i\rangle=1, \forall i \in \{1,2,\dots,l\}$. If $H_{ii}<1$, we can see that the norm $\|v_i\|=\sqrt{H_{ii}}<1$. Since $\|w_i\| \times \|v_i\|\geq \langle w_i,v_i\rangle=1$, we have $\|w_i\|>1$. Hence, $(H^{-1})_{ii}=\langle w_i,w_i\rangle=\|w_i\|^2 >1$.
\end{proof}

\section{Proof of Proposition \ref{prop:SBMCIF}} \label{appendix:SBMCIF}
\begin{proof}
Let $M_{vbm}$ be the set of all probability distributions realized by SBM. \cite{amari92igbm} proves that the mixed-coordinates of the resulting projection $P$ on $M_{vbm}$ is $[\zeta]_{P}=(\eta^1_i, \eta^2_{ij}, 0,\dots,0)$, given the 2-mixed-coordinates of $q(x)$. $M_{vbm}$ is equivalent to the submanifold tailored by CIF, i.e. $[\zeta]_{2_t}$. The corollary follows from Proposition \ref{prop:GeometricView}.
\end{proof}

\section{Proof of Proposition \ref{prop:sbmmlcloseform}} \label{appendix:sbmmlcloseform}
\begin{proof}
Based on Equation \ref{eq:thetaforBM}, the coordinates $[\theta_{2+}]$ for VBM is zero: $\theta_{2+}=0$.
Next, we show that the stationary distribution $p(x;\xi)$ learnt by ML has the same $[\eta_{i}^1, \eta_{ij}^2]$ with $q(x)$.

For VBM, the $\frac{\partial E(x;\xi)}{\partial \xi}$ can be easily calculated from Equation (\ref{eq:energyBM}):
$$\begin{cases} \frac{\partial E(x;\xi)}{\partial U_{x_i x_j}}=x_i x_j, & for ~ U_{x_i x_j}\in \xi; \\
\frac{\partial E(x;\xi)}{\partial b_{x_i}}=x_i, & for ~ b_{x_i}\in \xi. \end{cases}$$

Thus, based on Equation \ref{eq:learnrulestochasticgradient}, the gradients for $U_{x_i x_j},b_{x_i}\in \xi$ are as follows:
\begin{equation}\label{eq:gradientloglikelihoodsbm}
\begin{cases}
  \frac{\partial \log p(\underline{x};\xi)}{\partial U_{x_i,x_j}} = \langle x_i x_j \rangle_0 - \langle x_i x_j \rangle_\infty = \eta_{ij}^2(q(x))-\eta_{ij}^2(p(x;\xi)) \\
  \frac{\partial \log p(\underline{x};\xi)}{\partial b_{x_i}} = \langle x_i\rangle_0 - \langle x_i\rangle_\infty = \eta_{i}^1(q(x))-\eta_{i}^1(p(x;\xi))
\end{cases}\nonumber
\end{equation}
where $\langle\cdot \rangle_0$ denotes the average using the sample data and $\langle\cdot \rangle_\infty$ denotes the average with respect to the stationary distribution $p(x;\xi)$.

Since VBM defines an $e$-flat submanifold $M_{vbm}$ of $S$ \cite{amari92igbm}, then ML converges to the unique solution that gives the best approximation $p(x;\xi)\in M_{vbm}$ of $q(x)$. When ML converges, we have $\Delta \xi\rightarrow 0$ and hence $\frac{\partial \log p(\underline{x};\xi)}{\partial \xi}\rightarrow 0$. Thus, we can see that ML converges to stationary distribution $p(x;\xi)$ that preserves coordinates $[\eta_{i}^1, \eta_{ij}^2]$ of $q(x)$. This completes the proof.
\end{proof}

\section{Proof of Proposition \ref{prop:monotonicdivergence}} \label{appendix:monotonicdivergence}
\begin{proof}
Since $p_i\in B$ and $p_{i+1}\in B$ is the projection of $q_{i+1}$, then $D[q_{i+1},p_i] \geq D[q_{i+1},p_{i+1}]$. Similarly, $q_{i+1}\in H_q$ and $q_{i+2}\in H_q$ is the projection of $p_{i+1}$, thus $D[q_{i+1},p_{i+1}] \geq D[q_{i+2},p_{i+1}]$. This completes the proof.
\end{proof}

\section{Proof of Proposition \ref{prop:hqtorbm}} \label{appendix:hptorbm}
\begin{proof}
Based on the definition of divergence, the following relation holds:
\begin{eqnarray}\label{eq:divergenceqp}
      &&D[q(x,h),p(x,h)] = D[q(x)q(h|x),p(x)p(h|x)] \nonumber \\
      &=& E_{q(x,h)}[log\frac{q(x)}{p(x)}+log\frac{q(h|x)}{p(h|x)}] \nonumber \\
      &=& D[q(x),p(x)] + E_{q(x)}[D[q(h|x),p(h|x)]] \nonumber
\end{eqnarray}
where $E_{q(x,h)}[\cdot]$ and $E_{q(x)}[\cdot]$ are the expectations taken over $q(x,h)$ and $q(x)$ respectively.

Therefore, the minimum divergence between $p(x,h;\xi_p)$ and $H_q$ is given as:
\begin{eqnarray}\label{eq:mindivergencehb}
      &&\!\!\!\!\!\!D(H_q,p(x,h;\xi_p))=\!\!\!\! \min_{q(x,h;\xi_q)\in H_q} D[q(x,h;\xi_q),p(x,h;\xi_p)] \nonumber \\
      &=&\min_{\xi_q}\{ D[q(x),p(x)] + E_q(x)[D[q(h|x;\xi_q),p(h|x;\xi_p)]]\} \nonumber \\
      &=& D[q(x),p(x)] + \min_{\xi_q}\{ E_{q(x)}[D[q(h|x;\xi_q),p(h|x;\xi_p)]]\} \nonumber \\
      &=& D[q(x),p(x)] \nonumber
\end{eqnarray}
In the last equality, the expected divergence between $q(h|x;\xi_q)$ and $p(h|x;\xi_p)$ vanishes if and only if $\xi_q = \xi_p$. This completes the proof.\footnote{Note that a similar path of proof is also used in Theorem 7 of \cite{amari92igbm}. Here, we reformulate the proof to derive the projection $\Gamma_H(p(x,h))$.}
\end{proof}

\section{Proof of Proposition \ref{prop:projectionRBMcloseform}} \label{appendix:projectionRBMcloseform}
\begin{proof}
First, we prove the uniqueness of the projection $\Gamma_B(q)$. From the $[\theta]$ of BM in Equation (\ref{eq:thetaforBM}), $B$ is an $e$-flat smooth submanifold of $S_{xh}$. Thus the projection is unique.

Second, in order to find the $p(x,h;\xi_p)\in B$ with parameter $\xi_p$ that minimizes the divergence between $q(x,h;\xi_q)\in H_q$ and $B$, the gradient descent method iteratively adjusts $\xi_p$ in the negative gradient direction that the divergence $D[q,p(\xi_p)]$ decreases fastest:
\begin{equation}\label{eq:gradientdescentMinimumDivergence}
  \triangle \xi_p=-\lambda \frac{\partial D[q,p(\xi_p)]}{\partial \xi_p} \nonumber
\end{equation}
where $D[q,p(\xi_p)]$ is treated as a function of BM's parameters $\xi_p$ and $\lambda$ is the learning rate. As shown in \cite{highorderBM}, the gradient descent method converges to the minimum of the divergence with proper choices of $\lambda$, and hence achieves the projection point $\Gamma_B(q)$.

Last, we show that the mixed coordinates $[\zeta^{xh}]_{\Gamma_B(q)}$ in Equation (\ref{eq:mixedcoordinatenewProjectionRBM}) is exactly the convergence point of the ML learning for BM.
For distributions on the manifold $S_{xh}$, the states of all hidden units is also visible and hence the BM with hidden units is equivalent to VBM by treating hidden units as visible ones. Based on Proposition \ref{prop:sbmmlcloseform}, ML converges to the projection point $\Gamma_B(q)$ with a stationary distribution $p(x,h;\xi_p)$ that preserves coordinates $[\eta_{x_i}^1, \eta_{h_j}^1, \eta_{x_ix_j}^2, \eta_{x_ih_j}^2, \eta_{h_ih_j}^2]$ of $q(x,h;\xi_q)$.
This completes the proof.
\end{proof}


\end{document}